\documentclass[runningheads]{llncs}
\usepackage[misc]{ifsym}
\usepackage{graphicx}
\usepackage{amsmath}
\usepackage{amsfonts}
\usepackage{amssymb}
\usepackage{bbm}
\usepackage{mathtools}

\DeclareMathOperator*{\argmax}{arg\,max}
\def\dtv{d_{\mathrm{TV}}}
\def\gwide{\widehat{g}}

\def\bbX{\mathbb{X}}

\def\be{\mathbf{e}}
\def\mD{d}
\def\mX{\mathcal{X}}
\def\mY{\mathcal{Y}}
\def\mN{\mathcal{N}}
\def\mB{\mathcal{B}}
\def\bbR{\mathbb{R}}
\def\bbE{\mathbb{E}}
\def\bbP{\mathbb{P}}

\def\sbar{\bar{s}}
\def\smin{\mathrm{min}, s}

\def\bx{\mathbf{x}}
\def\bxs{\mathbf{x}_s}
\def\bxsb{\mathbf{x}_{\sbar}}
\def\bxsmin{\mathbf{x}_{\smin}}

\def\bz{\mathbf{z}}

\def\bX{\mathbf{X}}
\def\bXs{\mathbf{X}_s}
\def\bXsb{\mathbf{X}_{\sbar}}

\def\pX{p_{\bX}}
\def\ppX{p'_{\bX}}
\def\pXx{p_{\bX}(\bx)}
\def\pXsx{p_{\bXs}(\bxs)}

\def\pXsb{p_{\bX_{\sbar}}}
\def\ppXsb{p'_{\bX_{\sbar}}}
\def\pXsbx{p_{\bX_{\sbar}}(\bxsb)}
\def\ppXsbx{p'_{\bX_{\sbar}}(\bxsb)}
\def\pXsXt{p_{\bX_s | \bX_t}}
\def\pXsbXs{p_{\bXsb | \bXs = \bxs}}
\def\ppXsbXs{p'_{\bXsb | \bXs = \bxs}}
\def\pXsbXsz{p_{\bXsb | \bXs = z}}
\def\pXsbXszmax{p_{\bXsb | \bXs = z^{*}}}
\def\ppXsbXsz{p'_{\bXsb | \bXs = z}}
\def\ppXsbXszmax{p'_{\bXsb | \bXs = z^{*}}}

\spnewtheorem{assumption}[theorem]{Assumption}{\bfseries}{\itshape}
\usepackage{graphicx}
\usepackage[pagebackref]{hyperref}
\hypersetup{
    colorlinks,
    linkcolor={purple},
    citecolor={blue},
    urlcolor={blue!80!black}
}
\usepackage{doi} 
\usepackage{xcolor} 

\usepackage{booktabs} 
\usepackage{multirow} 
\usepackage[square,numbers,sort&compress]{natbib}  

\begin{document}

\title{On the Robustness of Global Feature Effect Explanations}

\author{Hubert Baniecki\inst{1} (\Letter) \and Giuseppe Casalicchio\inst{2,3}\and\\Bernd Bischl\inst{2,3} \and Przemyslaw Biecek\inst{1,4}}
\authorrunning{H. Baniecki et al.}
\institute{University of Warsaw, Poland, \texttt{h.baniecki@uw.edu.pl} \and LMU Munich, Germany \and Munich Center for Machine Learning (MCML), Germany \and Warsaw University of Technology, Poland}

\maketitle              

\begin{abstract}
We study the robustness of global post-hoc explanations for predictive models trained on tabular data. Effects of predictor features in black-box supervised learning are an essential diagnostic tool for model debugging and scientific discovery in applied sciences. However, how vulnerable they are to data and model perturbations remains an open research question. We introduce several theoretical bounds for evaluating the robustness of partial dependence plots and accumulated local effects. Our experimental results with synthetic and real-world datasets quantify the gap between the best and worst-case scenarios of (mis)interpreting machine learning predictions globally.

\end{abstract}

\section{Introduction}

Post-hoc explainability methods have become a standard tool for interpreting black-box machine learning (ML) models~\citep{baniecki2023grammar,bodria2023benchmarking,schwalbe2023comprehensive}. While the majority of popular explanation methods focus on the \emph{local} perspective of a particular prediction, e.g., feature attributions~\citep{ribeiro2016why,lundberg2017unified} and counterfactual explanations~\citep{guidotti2022stable,guyomard2023generating}, this paper focuses on the \emph{global} perspective, specifically feature effects like partial dependence plots~\citep{friedman2001greedy} and accumulated local effects~\citep{apley2020visualizing}. We observe a widespread adoption of global feature effect explanations in various scientific domains, including medicine~\citep{petch2022opening}, engineering~\citep{mangalathu2022machine}, microbiology~\citep{robertson2023gut}, and climate~\citep{kemter2023controls}.

Alongside the adoption of post-hoc explanations in scientific practice, work has appeared questioning the quality of explanations regarding stability and faithfulness~\citep{,bodria2023benchmarking} as measured with quantitative evaluation metrics~\citep{hedstrom2023quantus}. 
In fact, the quality of explanations often correlates with model performance~\citep{lakkaraju2020robust,jia2021studying}. 
Alarmingly, local explanations have been shown to be vulnerable to adversarial manipulations that exploit their well-known limitations~\citep{baniecki2024adversarial,noppel2024explainable}, e.g., sampling out-of-distribution~\citep{slack2020fooling,laberge2023fooling} or sensitivity to overparameterized models~\citep{ghorbani2019interpretation,huang2023safari}. In general, the limitations of ML related to \emph{robustness} are often unknown or overlooked in applied research~\citep[see examples listed in the introduction of][]{freiesleben2023beyond}. 

Most recently in \citep{lin2023robustness}, the authors studied the robustness of local feature attributions~\citep{ribeiro2016why,lundberg2017unified} to input and model perturbations. However, the potential limitations of global feature effects related to robustness are currently understudied. In~\citep{baniecki2022fooling}, a data poisoning attack on partial dependence is proposed to manipulate the visualizations and thus change the model's interpretation. In~\citep{gkolemis2023rhale}, a heterogeneity-aware estimator of accumulated local effects is proposed that automatically determines an optimal splitting of feature values to balance bias and variance. A responsible adoption of feature effects in practice requires a more in-depth analysis, which motivates our research question:
\begin{quote}
    \emph{How robust are global feature effect explanations to data perturbations and model perturbations?}
\end{quote}
To this end, we derive theoretical bounds for evaluating the robustness of partial dependence plots and accumulated local effects. 

The main contributions of our work are:
\begin{enumerate}
    \item We analytically quantify the robustness of \emph{global feature effects to data perturbations} (Theorems~\ref{th:pd-data} \& \ref{th:ale-data}). The theoretical bounds we derive give a better intuition about the factors influencing explanations and advance our general understanding of these explanation methods.
    \item We relate the robustness of \emph{global feature effects to model perturbations} to the robustness of \emph{local feature attributions} (Lemma~\ref{lem:pd-model}). Moreover, we extend this result and derive a new bound for accumulated local effects (Theorem~\ref{th:ale-model}).
    \item We perform experiments with real-world datasets concerning data poisoning and model parameter randomization tests to computationally quantify the robustness of global feature effects in practical applications.
\end{enumerate}

\section{Related work}

\subsection{Global feature effect explanations} 

Partial dependence plots were one of the first methods proposed to interpret black-box ML models like gradient boosting machines~\citep{friedman2001greedy}. They provide a simplified view of feature effects without decomposing interaction effects as in the case of functional ANOVA~\citep{hooker2007generalized}. Feature effects have a natural interpretation corresponding to feature importance~\citep{casalicchio2018visualizing}. In~\citep{apley2020visualizing}, authors propose accumulated local effects as an improvement to partial dependence, which corrects the estimation when features in data are correlated. Highly efficient estimation of accumulated local effects is possible when a model is differentiable~\citep{gkolemis2023dale}, e.g., in the case of neural networks. In~\citep{herbinger2022repid}, authors propose regional effect plots to correct explanations when feature interactions are present in the model. Most recently in this line of work, a heterogeneity-aware estimator of accumulated local effects was proposed that improves a naive bin-splitting of features~\citep{gkolemis2023rhale}.

Crucially, feature effect explanations are computed based on a fitted model (or learning algorithm) and the underlying data distribution~\citep{molnar2023relating}. 
In~\citep{baniecki2022fooling}, an adversarial attack on partial dependence that poisons the data is used to manipulate the interpretation of an explanation. In~\citep{muschalik2023ipdp}, the authors propose a method to estimate partial dependence under data shifts that impact a model in the case of incremental learning. Both works can be viewed as specific cases generalized by our theoretical analysis. 

Despite their inherent limitations, feature effect explanations are useful to interpret ML models in applied sciences, e.g., the effect of age and cholesterol on the probability of heart disease~\citep[][figure~6]{petch2022opening}, the effect of wall properties on shear strength~\citep[][figures~5~\&~6]{mangalathu2022machine}, the impact of aridity on flood trends~\citep[][figure~5b]{kemter2023controls}, which motivates our further study on their robustness.

\subsection{Robustness and stability of explanations} 

Robustness is a key concept in ML with often divergent meanings as discussed in~\citep{freiesleben2023beyond}. In~\citep{lakkaraju2020robust}, the authors propose robust global explanations in the form of linear and rule-based explanations (a.k.a. interpretable surrogate models). Recent work on interpretability studies the robustness, also referred to as stability, of local explanation methods such as counterfactuals~\citep{guidotti2022stable,guyomard2023generating} and feature attributions~\citep{gan2022is,mayer2023minimizing,lin2023robustness}. Stability is defined in the literature as \emph{ability to provide similar explanations on similar instances}~\citep{guidotti2022stable,gan2022is} or \emph{obtaining similar explanations after the model is retrained on new data}~\citep{mayer2023minimizing}, which directly refers to the notion of \emph{robustness of explanations to input data and model perturbations}~\citep{lakkaraju2020robust,guyomard2023generating,lin2023robustness}. 

Related to this notion of robustness from the point of safety are works proposing adversarial attacks on explanation methods. 
Explanations can be manipulated by substituting a model~\citep{slack2020fooling}, crafting adversarial examples~\citep{ghorbani2019interpretation,huang2023safari} or poisoning the data~\citep{baniecki2022fooling,laberge2023fooling}. 
Such threats undermine the trustworthiness of ML systems that possibly cannot provide actionable explanations in real-world~\citep{guidotti2022stable,mayer2023minimizing}. In~\citep{wicker2023robust}, authors propose a method to guarantee the adversarial robustness of local gradient-based explanations. 

Our theoretical work directly relates to the robustness results obtained in~\citep{lin2023robustness} but instead aims at global feature effect explanations. 

\section{Notation and definition of feature effects}

We consider a supervised ML setup for regression or binary classification with labeled data $\{(\bx^{(1)},y^{(1)}), \ldots, (\bx^{(n)}, y^{(n)})\}$, where every element comes from $\mX \times \mY$, the underlying feature and label space. 
Usually, we assume $\mX \subseteq \bbR^{p}$.
We denote the $n \times p$ dimensional design matrix by $\bbX$ where $\bx^{(i)}$ is the $i$-th row of $\bbX$.
This data is assumed to be sampled in an i.i.d. fashion from an underlying distribution defined on $\mX \times \mY$.
We denote a random variable vector as $\bX \in \mX$ and the random variable for a label by $Y \in \mY$.
Let $s \subset \{1, \dots, p\}$ be a feature index set of interest with its complement $\sbar = \{1, \dots, p\} \setminus s$.
We often index feature vectors, random variables, and design matrices by index sets $s$ to restrict them to these index sets. 
We write $f(\bxs, \bxsb)$, to emphasize that the feature vector is separated into two parts, usually the ``features of interest'' $\bxs$ and the ```rest'' $\bxsb$.
We use $p(\bx, y)$ for the joint probability density function on $\mX \times \mY$, and
we write $\pXx$ and $\pXsx$ for marginal distributions for $\bX$ and $\bX_s$, respectively, and $\pXsXt(\bx_s | \bx_t)$ for the conditional distribution of  $\bX_s|\bX_t$.

We denote a prediction model by $f: \mX \mapsto \bbR$; it predicts an output using an input feature vector $\bx$.
In the case of binary classification, the output is either a decision score from $\bbR$ or a posterior probability from $[0,1]$. Without loss of generality, we explain the output for a single class in the case of a multi-class task.
Later, we will make changes to the design matrix of our labeled data on which we train a model.
To make this explicit, we sometimes write $f_{\bbX}$ to denote a model trained on $\bbX$ (and we suppress labels in notation here) but we will simply write $f$ when training data is clear from context. 

Furthermore, let $g(\cdot\;;f,\pX)$ denote a general explanation function where we emphasize in notation that the explanation depends both on a given model $f$ and the feature density $\pX$ -- which we will both perturb later. Specifically, let $g_s(\bx_s; f, \pX) $ denote a global feature effect (e.g., a $\textsc{pd}_s$ or $\textsc{ale}_s$ function as defined in Definitions \ref{def:pd} and \ref{def:ale}) for a set of features of interest $s$ -- often $|s| = 1$ when the effect is visualized as a line curve or $|s| = 2$ when it is visualized as a heat map. In practice, an estimator of feature effects denoted by $\gwide_s(\bx_s; f, \bbX)$  requires estimating probability density $\pX$ using particular input data $\bbX$. 

\begin{definition}[Partial dependence \citep{friedman2001greedy}] Partial dependence for feature set $s$ is defined as 
$
    \textsc{pd}_s(\bxs; f, \pX) =  
    \bbE_{\bXsb \sim \pXsb} \left[f(\bxs, \bXsb)\right] = 
    \int f(\bxs, \bxsb) \pXsbx d\bxsb,
$
which in practice can be estimated using Monte-Carlo estimation:
$
    \widehat{\textsc{pd}}_s(\bxs; f, \bbX) = 
    \frac{1}{n} \sum_{i=1}^{n} \left[f(\bxs, \bxsb^{(i)})\right].
$
\label{def:pd}
\end{definition}

\begin{definition}[Conditional dependence, i.e., Marginal plot \citep{friedman2001greedy,apley2020visualizing}] 
Conditional dependence is defined as
$
    \textsc{cd}_s(\bxs; f, \pX) = \bbE_{\bXsb \sim \pXsbXs} \left[f(\bxs, \bXsb)\right] = \int f(\bxs, \bxsb) \pXsbXs(\bxsb | \bxs) d\bxsb,
$
which can be estimated using
$
    \widehat{\textsc{cd}}_s(\bxs; f, \bbX) = \frac{1}{|\mN(\bxs)|} \sum_{i \in \mN(\bxs)} \left[f(\bxs, \bxsb^{(i)})\right], 
$
where $\mN(\bxs) \coloneqq \{i: \| \bxs - \bxs^{(i)} \|  \leq  \epsilon\}$ denotes indices of observations in an $\epsilon$-neighborhood of $\bxs$ for a given norm $\| \cdot \|$.
\end{definition}

\begin{definition}[Accumulated local effects \citep{apley2020visualizing}] For brevity, we define here $\textsc{ale}_s$ for case when $|s|=1$ as
\begin{equation}
\begin{split}
   \textsc{ale}_s(\bxs; f, \pX) 
   & = \int_{\bxsmin}^{\bxs} \bbE_{\bXsb \sim \pXsbXsz}\left[\frac{\partial f(z, \bXsb)}{\partial z}\right] dz \\
   & = \int_{\bxsmin}^{\bxs} \int \left[\frac{\partial f(z, \bxsb)}{\partial z}\right] \pXsbXsz(\bxsb | z) d\bxsb dz,
\end{split}
\end{equation}
where $\bxsmin$ is a value chosen near the lower bound of the support of feature~$s$. $\textsc{ale}_s$ can be estimated using
\begin{equation}
   \widehat{\textsc{ale}}_s(\bxs; f, \bbX) = \sum_{k=1}^{k_{\bxs}} \frac{1}{|\mB(k)|} \sum_{\mB(k)} \left[f(z^k, \bxsb^{(i)}) - f(z^{k-1}, \bxsb^{(i)})\right], 
\end{equation}
where $(z^{1}, \ldots, z^{k_{\bxs}})$ are grid points spanning the domain of feature $s$ up to $\bx_s$ and $\mB(k) \coloneqq \{i:\bxs^{(i)} \in (z^{k-1}, z^k)\}$ are indices of observations in a $k$-th bin.
\label{def:ale}
\end{definition}

Otherwise, $\widehat{\textsc{dale}}_s$~\citep{gkolemis2023dale} is a more efficient estimator of $\textsc{ale}_s$ under the assumption that $f$ is theoretically differentiable, and auto-differentiable in practice, e.g., a neural network. Both $\widehat{\textsc{ale}}_s$ and $\widehat{\textsc{dale}}_s$ are estimated up to a constant (also called an ``uncentered'' estimator), which in practice is corrected by adding to it the mean prediction of the model.

Feature effect explanation is usually estimated on a finite set of grid points $\bz = (z^{1}, \ldots, z^{m})$ spanning the domain of feature $\bx_s$. 
The most popular choices for grid points are quantile values or an equidistant grid. 
Function $\gwide$ then returns an estimated $m$-dimensional explanation vector $\be$ defined on the grid $\bz$, and the curve can be visualized from the finite set of points $\{z^k,\;\gwide_s(z^k \mid f, \bbX)\}_{k=1}^{m}$.

\section{Theoretical analysis}

Let the symbol $\rightarrow$ denote a change in a given object, e.g., a small perturbation in the data. In general, our goal is to quantify the change in \textbf{global} explanations $\be \rightarrow \be'$ in means of data change $\bbX \rightarrow \bbX'$ (e.g., distribution shift) or model change $f \rightarrow f'$ (e.g., fine-tuning) measured with some distance function $\mD$. 
In the case of model-agnostic interpretability, typically both model and data are used as input to the explanation estimator $\gwide(\cdot \;; f_{\bbX},\bbX)$. 
The literature considers the following scenarios for analyzing explanation robustness:
\begin{enumerate}
    \item[\ref{sec:data}] \textbf{data perturbation} when $\bbX \rightarrow \bbX'$ implies $ \gwide(\cdot\;; f_{\bbX},\bbX) \rightarrow \gwide(\cdot\;; f_{\bbX},\bbX')$, also known as data poisoning \citep{baniecki2022fooling} or biased sampling \citep{laberge2023fooling},
    \item[\ref{sec:model}] \textbf{model perturbation} when $f \rightarrow f'$ implies $\gwide(\cdot\;; f, \bbX) \rightarrow \gwide(\cdot\;; f', \bbX)$, which, in practice, often corresponds to either 
    \begin{itemize}
        \item $\bbX \rightarrow \bbX'$ implies $\gwide(\cdot\;; f_{\bbX},\bbX) \rightarrow \gwide(\cdot\;; f_{\bbX'},\bbX)$ in case of data shifts~\citep{mayer2023minimizing}, or
        \item $\bbX \rightarrow \bbX'$ implies $\gwide(\cdot\;; f_{\bbX},\bbX) \rightarrow \gwide(\cdot\;; f_{\bbX'},\bbX')$ in incremental learning~\citep{muschalik2023ipdp}.
    \end{itemize}
\end{enumerate}
Therefore, quantifying \emph{robustness} can be defined as quantifying bounded relationships between explanation change $\mD(\be,\be')$ and data change $\mD(\bbX,\bbX')$ or model change $\mD(f,f')$. We do exactly that for global feature effects.

\subsection{Robustness to data perturbation}\label{sec:data}

Consider a simple scenario where the model function is the XOR function of two features
$f(\bx_1, \bx_2) = \mathbbm{1}_{\bx_1\cdot \bx_2 > 0}.$ We want to explain an effect of feature $\bx_1$ so taking a partial dependence of $f$ on $\bx_1$ yields
\begin{equation}
\begin{split}
\textsc{pd}_1(\bx_1 ; f, \pX) & = \bbE_{\bX_2 \sim p_{\bX_2}} \left[f(\bx_1, \bX_2)\right] 
= \bbE_{\bX_2 \sim p_{\bX_2}} \left[\mathbbm{1}_{\bx_1\cdot \bx_2 > 0}\right] \\
& = \bbP(\bx_1\cdot \bX_2 > 0) =
\begin{cases}
\bbP(\bX_2>0), & \text{if } \bx_1 > 0\\
\bbP(\bX_2<0), & \text{otherwise}.
\end{cases}
\end{split}
\end{equation}
We observe that an explanation of $\bx_1$ depends solely on the distribution~$p_{\bX_2}$. 

For example, assuming the distribution of the second feature is given by $\bX_2 \sim \mathcal{U}[a,b] \text{ where } a\leq0\leq b$, we have
\begin{equation}
\textsc{pd}_1(\bx_1 ; f, \pX) =
\begin{cases}
\bbP(\bX_2>0) = \frac{b}{b-a}, & \text{if } \bx_1 > 0\\
\bbP(\bX_2<0) = \frac{a}{b-a}, & \text{otherwise}.
\end{cases}
\end{equation}
Figure~\ref{fig:simple} shows this relationship between a feature effect in grid point $\bx_1 = 1$ and perturbing distribution $p_{\bX_2}$ computationally; also for a normal distribution. 

\begin{figure}[t]
    \centering
    \includegraphics[width=0.49\textwidth]{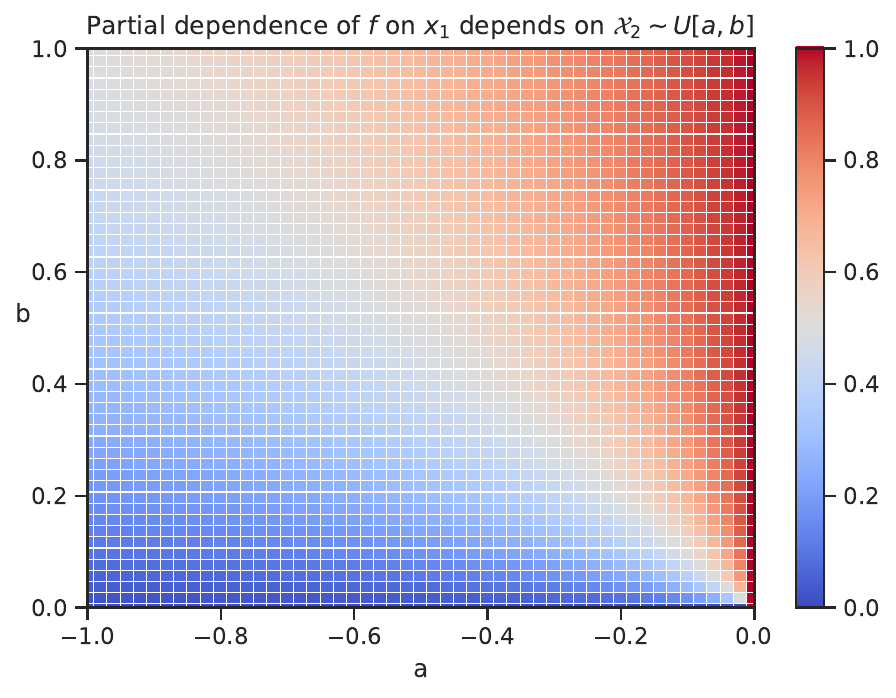}
    \includegraphics[width=0.49\textwidth]{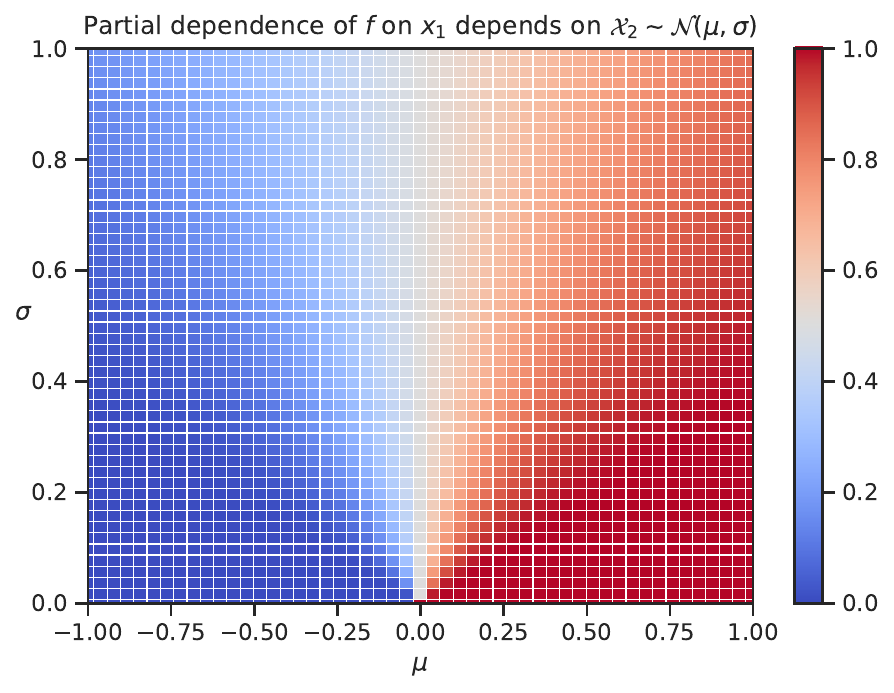}
    \caption{Global feature effect explanation (in color) for value $\bx_1 = 1$ depends \underline{solely} on the parameters of a uniform (\textbf{left}) or normal (\textbf{right}) distribution of feature $\bX_2$.}
    \label{fig:simple}
\end{figure}

We are interested in finding a theoretical bound for this relationship in a general case.

\begin{assumption}\label{ass:bound}
We assume that the model $f$ has bounded predictions, i.e., there exists a constant $B$ such that $|f(\bx)| \leq B$ for all $\bx \in \bbR^p$.
\end{assumption}

\begin{theorem} \label{th:pd-data}
The robustness of partial dependence and conditional dependence to \textbf{data perturbations} is given by the following formulas
\begin{equation}
    \big|\textsc{pd}_s(\bxs; f, \pX) - \textsc{pd}_s(\bxs; f, \ppX)\big| \leq  2B \cdot \dtv\big(\pXsb, \ppXsb\big),
\end{equation}
\begin{equation}
    \big|\textsc{cd}_s(\bxs; f, \pX) - \textsc{cd}_s(\bxs; f, \ppX)\big|
    \leq 2B \cdot \dtv\big(\pXsbXs, \ppXsbXs\big),
\end{equation}
where the total variation distance $\dtv$ is defined via the $l_1$ functional distance.
\end{theorem}
\begin{proof} We have
{\allowdisplaybreaks
\begin{align*}
    & \big|\textsc{pd}_s(\bx_s; f, \pX) - \textsc{pd}_s(\bx_s; f, \ppX)\big| = \\
    & = \left|\int f(\bxs, \bxsb) \pXsbx d\bxsb - \int f(\bxs, \bxsb) \ppXsbx d\bxsb\right| \\
    & = \left|\int f(\bxs, \bxsb) \big(\pXsbx - \ppXsbx\big) d\bxsb \right| \\
    & \leq \int \big| f(\bxs, \bxsb) \big(\pXsbx - \ppXsbx\big) \big| d\bxsb \\
    & = \int \big|f(\bxs, \bxsb)\big| \cdot \left|\pXsbx - \ppXsbx\right| d\bxsb \\
    & \text{from Assumption~\ref{ass:bound}, we have} \\
    & \leq \int B \cdot \left|\pXsbx - \ppXsbx\right| d\bxsb \\ 
    & = B \cdot \int \left|\pXsbx - \ppXsbx\right|\; d\bxsb \\ 
    & = 2B \cdot \dtv\big(\pXsb, \ppXsb\big).
\end{align*}
}
We now apply this argument again, with expected value $\bbE_{\bXsb \sim \pXsbXs}$ instead of $\bbE_{\bXsb \sim \pXsb}$, to obtain
\begin{equation}
    \big|\textsc{cd}_s(\bxs; f, \pX) - \textsc{cd}_s(\bxs; f, \ppX)\big|
    \leq 2B \cdot \dtv\big(\pXsbXs, \ppXsbXs\big).
\end{equation}
\qed
\end{proof}

Theorem~\ref{th:pd-data} gives an upper bound on the possible change of global feature effects in terms of distance between data distributions given model-specific constant $B$, e.g. 1 in classification or a maximum value of target domain in regression. Certainly in many scenarios, an average value of model prediction $\big|f(\bxs, \bxsb)\big|$, and hence the bound, is smaller. 

\begin{remark}
    In Theorem~\ref{th:pd-data}, we can obtain a tighter bound per point $\bxs$ by taking $B(\bxs)$ such that  $|f(\bxs, \bxsb)| \leq B(\bxs)$ for all $\bxsb \in \bbR^{p-|s|}$.
\end{remark}

\begin{remark}\label{remark:constant}
    By definition, if $f$ has bounded predictions such that $A \leq f(\bx) \leq B$, then the feature effect value is bounded, i.e., $A \leq g_s(\bxs; f, \pX) \leq B$. From this follows that a change in the feature effect value will be smaller than the maximal distance to these bounds ($A$ or $B$). We obtain
    $$
    \big|g_s(\bxs; f, \pX) - g_s(\bxs; f, \ppX)\big| \leq \max(|g_s(\bxs; f, \pX) - B|, |g_s(\bxs; f, \pX) - A|),
    $$ 
    which makes the bound constant for high-enough values of $\dtv\big(\pX, \ppX\big)$. For example, taking
    $A = 0$ and $B = 1$, for $\bxs$ such that $g_s(\bxs; f, \pX) = 0.5$, we have 
    \begin{equation}
    \big|g_s(\bxs; f, \pX) - g_s(\bxs; f, \ppX)\big| \leq
    \begin{cases}
    2 \cdot \dtv\big(\pX, \ppX\big), & \text{if } \dtv\big(\pX, \ppX\big) \leq 0.25 ,\\
    0.5, & \text{otherwise}.
    \end{cases}
    \end{equation}
\end{remark}

We conduct analogous theoretical analysis for accumulated local effects based on the following assumption about the model $f$, which holds for many predictive functions, e.g., typical neural networks~\citep{virmaux2018lipschitz}. 
\begin{assumption}\label{ass:lipschitz}
We assume that the model $f$ is globally $L$-Lipschitz continuous, or that we have $\|f(\bx) - f(\bx')\| \leq L \cdot \|\bx - \bx'\|$ for all $\bx, \bx' \in \bbR^p$.
\end{assumption}
It can be easily shown that Assumption~\ref{ass:lipschitz} leads to Lemma~\ref{lem:lipschitz}.
\begin{lemma}\label{lem:lipschitz}
If $f$ is $L$-Lipschitz, then it holds that $\|\nabla f\| \leq L$ almost everywhere. 
\end{lemma}

\begin{theorem} \label{th:ale-data}
The robustness of accumulated local effects to \textbf{data perturbations} is given by the following formula
\begin{equation}
    \big|\textsc{ale}_s(\bxs; f, \pX) - \textsc{ale}_s(\bxs; f, \ppX)\big|
    \leq 2L \cdot (\bxs - \bxsmin) \cdot \dtv\big(\pXsbXszmax, \ppXsbXszmax\big),
\end{equation}
where $z^{*} = \argmax_z \dtv\big(\pXsbXsz, \ppXsbXsz\big)$.
\end{theorem}
\begin{proof} We have
\begin{equation}
\begin{split}
    & \big|\textsc{ale}_s(\bxs; f, \pX) - \textsc{ale}_s(\bxs; f, \ppX)\big| = \\
    & = \Bigg| \int_{\bxsmin}^{\bxs} \int \frac{\partial f(z, \bxsb)}{\partial z}  \pXsbXsz(\bxsb | z) d\bxsb dz \\
    & - \int_{\bxsmin}^{\bxs} \int \frac{\partial f(z, \bxsb)}{\partial z} \ppXsbXsz(\bxsb | z) d\bxsb dz \Bigg| \\
    & = \left| \int_{\bxsmin}^{\bxs} \int \frac{\partial f(z, \bxsb)}{\partial z}  \big(\pXsbXsz(\bxsb | z) - \ppXsbXsz(\bxsb | z)\big) d\bxsb dz \right| \\
    & \leq \int_{\bxsmin}^{\bxs} \int \Big| \frac{\partial f(z, \bxsb)}{\partial z}  \big(\pXsbXsz(\bxsb | z) - \ppXsbXsz(\bxsb | z)\big)\Big| d\bxsb dz \\
    & = \int_{\bxsmin}^{\bxs} \int \Big| \frac{\partial f(z, \bxsb)}{\partial z} \Big|  \cdot \Big|\pXsbXsz(\bxsb | z) - \ppXsbXsz(\bxsb | z)\Big| d\bxsb dz \\
    & \text{from Assumption~\ref{ass:lipschitz} and Lemma~\ref{lem:lipschitz}, we have} \\
    & \leq \int_{\bxsmin}^{\bxs} L \int \left|\pXsbXsz(\bxsb | z) - \ppXsbXsz(\bxsb | z)\right| d\bxsb dz \\
    & = \int_{\bxsmin}^{\bxs} 2L \cdot \dtv\big(\pXsbXsz, \ppXsbXsz\big) dz \\
    & \leq 2L \cdot (\bxs - \bxsmin) \cdot \max_z \dtv\big(\pXsbXsz, \ppXsbXsz\big).
\end{split}
\end{equation}
\qed
\end{proof}

Theorem~\ref{th:ale-data} is for case when $|s| = 1$, but it can be generalized to $|s| > 1$. The robustness of accumulated local effects differs from that of partial and conditional dependence as $\textsc{ale}$ uses a gradient of a function bounded by $L$ instead of the model's prediction bounded by $B$. In general, estimating $L$ is not obvious, but in many cases it can be found computationally~\citep[see a discussion in][]{virmaux2018lipschitz}.

\begin{remark}
    Interestingly, the ``accumulated'' nature of the method makes the bound $2L \cdot (\bxs - \bxsmin) \cdot \dtv\big(\pXsbXszmax, \ppXsbXszmax\big)$ an increasing function of $\bxs$, while the corresponding bound of $\textsc{pd}_s$ and $\textsc{cd}_s$ does not posses such property.
\end{remark}

We further support the derived theory concerning the robustness of global feature effects to data perturbations with experimental results in Section~\ref{sec:experiments}.

\subsection{Robustness to model perturbation}
\label{sec:model}

In this section, we first relate the robustness of global feature effects to the robustness of local removal-based feature attributions discussed in~\citep{lin2023robustness} (Lemma~\ref{lem:pd-model}). Theorem~\ref{th:ale-model} extends these results to accumulated local effects. 
 
\begin{lemma}\label{lem:pd-model} 
The robustness of partial dependence and conditional dependence to \textbf{model perturbations} is given by the following formulas
\begin{equation}
    \big|\textsc{pd}_s(\bxs; f, \pX) - \textsc{pd}_s(\bxs; f', \pX)\big| 
    \leq \|f-f'\|_{\infty},
\end{equation}
\begin{equation}
    \big|\textsc{cd}_s(\bxs; f, \pX) - \textsc{cd}_s(\bxs; f', \pX)\big| 
    \leq \|f-f'\|_{\infty, \mX},
\end{equation}
where $\| f \|_\infty \coloneqq \sup_{\bx \in \bbR^p} |f(\bx)| $ denotes an infinity norm for a function and $\| f \|_{\infty, \mX} \coloneqq \sup_{\bx \in \mX} |f(\bx)|$ is the same norm taken over the domain $\mX \subseteq \bbR^p$.
\end{lemma}
\begin{proof}
    Follows directly from \citep[][lemmas 5 \& 6]{lin2023robustness}.
\end{proof}

\begin{theorem} \label{th:ale-model} 
The robustness of accumulated local effects to \textbf{model perturbations} is given by the following formula
\begin{equation}
    \big|\textsc{ale}_s(\bxs; f, \pX) - \textsc{ale}_s(\bxs; f', \pX)\big| \leq (\bxs - \bxsmin) \cdot \|h - h'\|_{\infty, \mX},
\end{equation}
where $h \coloneqq \frac{\partial f}{\partial \bxs}$ and $h' \coloneqq \frac{\partial f'}{\partial \bxs}$ denote partial derivatives of $f$ and $f'$ respectively.
\end{theorem}
\begin{proof} We have
{\allowdisplaybreaks
\begin{align*}
    & \big|\textsc{ale}_s(\bxs; f, \pX) - \textsc{ale}_s(\bxs; f', \pX)\big| = \\
    & = \Bigg| \int_{\bxsmin}^{\bxs} \int \frac{\partial f(z, \bxsb)}{\partial z}  \pXsbXsz(\bxsb | z) d\bxsb dz \\
    & - \int_{\bxsmin}^{\bxs} \int \frac{\partial f'(z, \bxsb)}{\partial z} \pXsbXsz(\bxsb | z) d\bxsb dz \Bigg| \\
    & = \left| \int_{\bxsmin}^{\bxs} \int \left(\frac{\partial f(z, \bxsb)}{\partial z} - \frac{\partial f'(z, \bxsb)}{\partial z}\right) \pXsbXsz(\bxsb | z) d\bxsb dz \right| \\
    & \leq \int_{\bxsmin}^{\bxs} \int \Big|\frac{\partial f(z, \bxsb)}{\partial z} - \frac{\partial f'(z, \bxsb)}{\partial z}\Big| \pXsbXsz(\bxsb | z) d\bxsb dz = (\star) \\
    & \text{We can derive two bounds:} \\
    & \text{(\textbf{A}) assuming $f'$ is globally $L'$-Lipschitz continuous, we have} \\
    & (\star) \leq (L + L') \cdot \int_{\bxsmin}^{\bxs} \int \pXsbXsz(\bxsb | z) d\bxsb dz \\
    & = (L + L') \cdot \int_{\bxsmin}^{\bxs} 1 dz = (\bxs - \bxsmin) \cdot (L + L') \\
    & \text{(\textbf{B}) substituting $h(z, \bxsb) \coloneqq \frac{\partial f(z, \bxsb)}{\partial z}$ and $h'(z, \bxsb) \coloneqq \frac{\partial f'(z, \bxsb)}{\partial z}$, we have} \\
    & (\star) = \int_{\bxsmin}^{\bxs} \int \big|h(z, \bxsb) - h'(z, \bxsb)\big| \pXsbXsz(\bxsb | z) d\bxsb dz \\
    & \text{from \textbf{Lemma \ref{lem:pd-model}}, we use $\big|\textsc{cd}_s(z; h, \pX) - \textsc{cd}_s(z; h', \pX)\big| \leq \|h - h'\|_{\infty, \mX}$ to obtain} \\
    & \leq \int_{\bxsmin}^{\bxs} \|h - h'\|_{\infty, \mX}\;dz  = (\bxs - \bxsmin) \cdot \|h - h'\|_{\infty, \mX}.
\end{align*}
}
We observe that the bound obtained in (\textbf{A}) is a specific worst-case scenario of the bound obtained in (\textbf{B}). From Lemma~\ref{lem:lipschitz}, it can be easily shown that $\|h - h'\|_{\infty, \mX} \leq (L + L')$ and so (\textbf{B}) is a tighter bound.
\qed
\end{proof}

Theorem~\ref{th:ale-model} is for case when $|s| = 1$, but it can be generalized to $|s| > 1$. The robustness of accumulated local effects to model perturbation differs from that of partial and conditional dependence as $\textsc{ale}$ is bounded by the norm between partial derivative functions $h$ instead of the model functions $f$.

\section{Experiments}\label{sec:experiments}

We provide additional empirical results supporting our theoretical analysis concerning the robustness of feature effects to data perturbation (Section~\ref{sec:experiments-data}) and model perturbation (Section~\ref{sec:experiments-model}). We rely on data and pretrained models from the OpenXAI benchmark~\citep{agarwal2022openxai} to make our experiments reproducible and to minimize bias related to the choice of a particular ML algorithm or a dataset preprocessing pipeline. Code to reproduce our experiments is available on GitHub at \url{https://github.com/hbaniecki/robust-feature-effects}.

\subsection{Robustness to data perturbation}\label{sec:experiments-data}

First, we aim to computationally analyze the relationship of how changes in the input data $\mD(\bbX,\bbX')$ affect changes in the resulting explanations $\mD(\be,\be')$. 

\paragraph{Setup.} To this end, we rely on the three datasets from OpenXAI that contain only continuous features: HELOC ($n=9871$, $p=23$) where the task is to predict whether a credit will be repaid in 2 years, Pima ($n=768$, $p=9$) where the task is to predict whether a patient has diabetes, and a Synthetic dataset (aka Gaussian, $n=5000$, $p=20$). We leave considerations concerning the perturbation of categorical features for future work. To each dataset, there is a pretrained neural network with an accuracy of 74\%~(HELOC), 92\%~(Synthetic), and 77\%~(Pima) that outperforms a logistic regression baseline~(72\%, 83\%, 66\%, respectively). We explain a neural network on the test sets of HELOC and Synthetic based on the pre-defined splits, but on the train set of Pima, as its test set is too small to reliably estimate conditional distributions and effects. 

To analyze a wide spectrum of possible feature effect explanations, we explain three features $s$: the least, ``median'' and most important to the model. We measure the importance of features with the variance of feature effects, i.e., higher variance means higher importance~\citep{zien2009feature,br2018simple}. For each feature, we evaluate $\bxs$ values on a grid of the three quantile values: 0.2, 0.5, and 0.8.

\paragraph{Data perturbation.} We perturb data using two approaches: a baseline of applying Gaussian noise with varying intensity $\mN(0,\sigma)$, and an adversarial perturbation found using a genetic algorithm proposed in~\citep{baniecki2022fooling}. The latter perturbs a dataset $\bbX$ used for estimating $\pX$ so that the change in an explanation value $|g_s(\bxs;f,\pX) - g_s(\bxs;f,\ppX)|$ is maximized for each $\bxs$ separately. Further details concerning methods and their hyperparameters are in Appendix~\ref{app:experiments-data}. In each scenario, we measure the magnitude of perturbation by estimating total variation distance $\dtv(\pX, \ppX)$ for partial dependence and $\dtv\big(\pXsbXs, \ppXsbXs\big)$ for conditional dependence.

\begin{figure}[b!]
    \centering
    \includegraphics[width=0.9\textwidth]{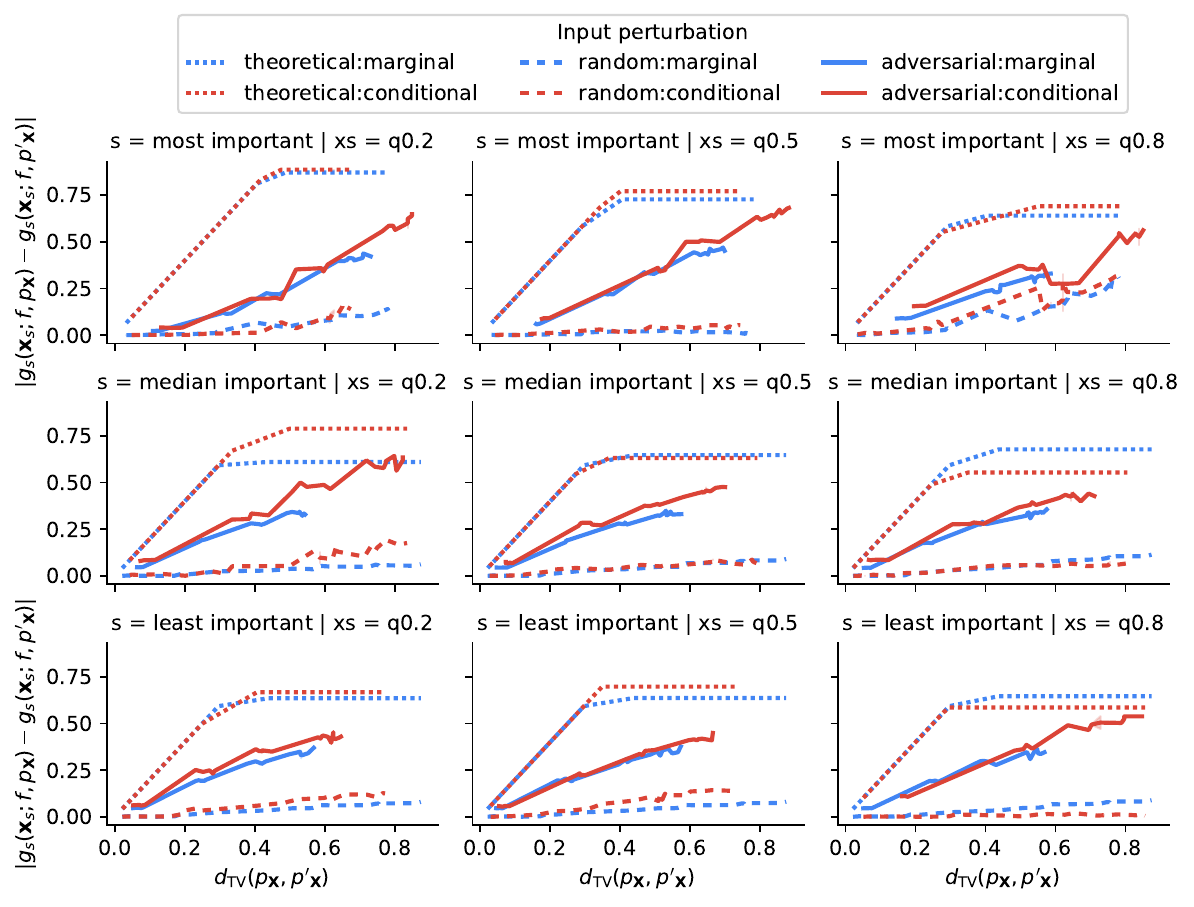}
    \caption{Robustness of feature effects to data perturbation for a neural network trained on the \textbf{Pima} dataset. Feature effect is evaluated for a feature $s$ (rows) and quantile value $\bxs$ (columns). The dotted lines denote the approximated theoretical bound derived in Theorem~\ref{th:pd-data}, which ends with a case-specific constant value as described in Remark~\ref{remark:constant}. Results are with respect to estimating feature effects on marginal $\pX$ (partial dependence) or conditional distribution $\pXsbXs$ (conditional dependence). }
    \label{fig:exp1_pima}
\end{figure}

\paragraph{Results.} Figure~\ref{fig:exp1_pima} shows results for the Pima dataset. In most cases, it is possible to adversarially perturb input data to drastically change the explanation. Note that our theoretical bounds are only a worst-case scenario analysis that might not occur in practice. Moreover, the adversarial algorithm might not find the best solution to the optimization problem of maximizing the distance between explanations. Figure~\ref{fig:exp1_synthetic} shows results for the Synthetic dataset, where we observe that even with simple predictive tasks it can be hard to significantly manipulate feature effects. More analogous results for the HELOC dataset are in Appendix~\ref{app:experiments-data}. On average in our setup, $\textsc{pd}_s$ (marginal) is more robust to data perturbation than $\textsc{cd}_s$ (conditional). 

\begin{figure}[ht]
    \centering
    \includegraphics[width=0.9\textwidth]{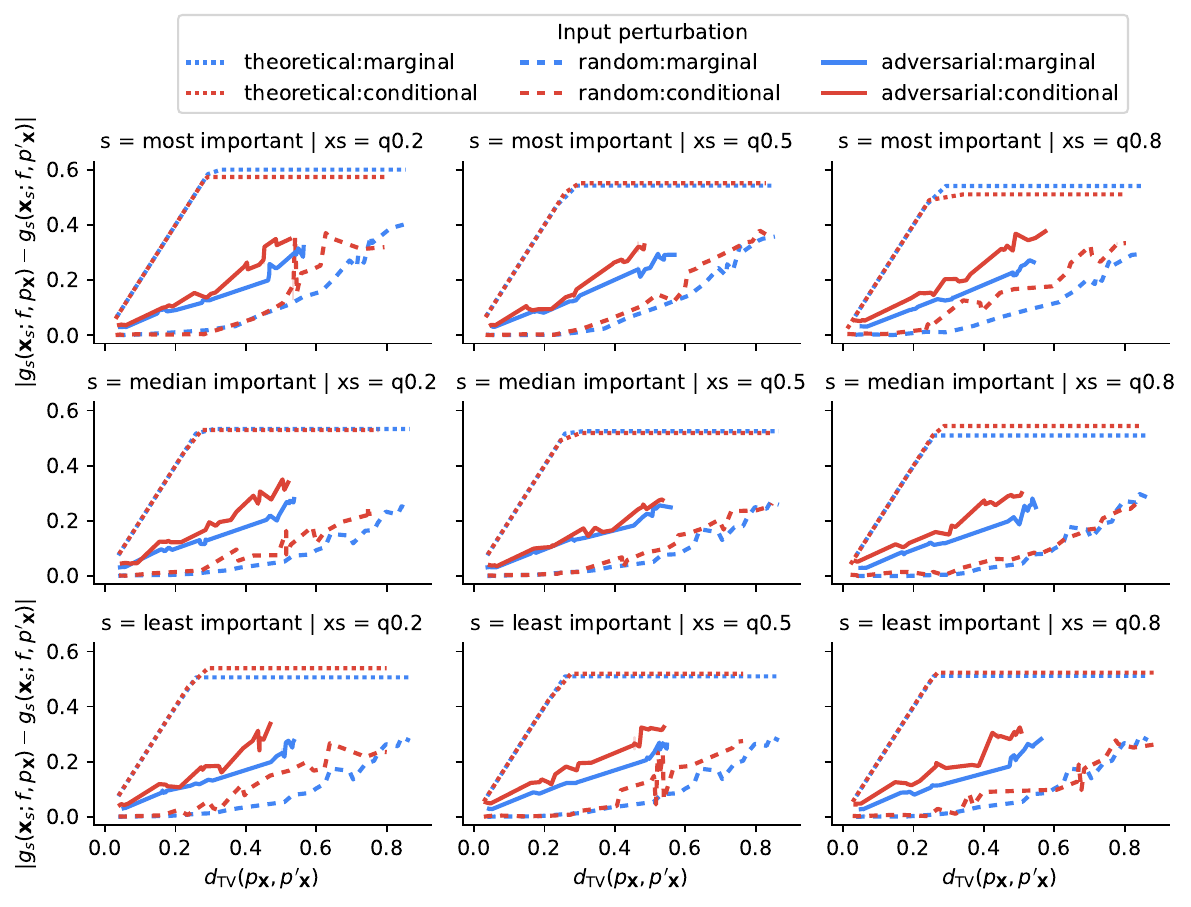}
    \caption{Robustness of feature effects to data perturbation for a neural network trained on the \textbf{Synthetic} dataset. Feature effect is evaluated for a feature $s$ (rows) and quantile value $\bxs$ (columns). The dotted lines denote the approximated theoretical bound derived in Theorem~\ref{th:pd-data}, which ends with a case-specific constant value as described in Remark~\ref{remark:constant}. Results are with respect to estimating feature effects on marginal $\pX$ or conditional distribution $\pXsbXs$. On average, partial dependence~(marginal) is more robust to data perturbation than conditional dependence.
    }
    \label{fig:exp1_synthetic}
\end{figure}

\subsection{Robustness to model perturbation}\label{sec:experiments-model}

Next, we aim to computationally analyze the relationship of how changes in the model parameters $\mD(f,f')$ affect changes in the resulting explanations $\mD(\be,\be')$. 

\paragraph{Setup.} We add to experiments the remaining datasets from OpenXAI that can include categorical features: Adult ($n=48842$, $p=13$) where the task is to predict whether an individual’s income exceeds \$50K per year, Credit (aka German, $n=1000$, $p=20$) where the task is to distinguish between a good or bad credit score, and Heart ($n=4240$, $p=16$) where the task is to predict whether the patient has a 10-year risk of future heart disease. To each dataset, there is a pretrained neural network with an accuracy of 85\%, 75\%, and 85\% respectively. We excluded the COMPAS dataset as the pretrained neural network does not outperform a logistic regression baseline (85.4\%), signaling that the model is underfitted, which might influence its robustness analysis. In this experiment, we want to aggregate results across all features and their values for three feature effects: marginal (partial, $\textsc{pd}_s$), conditional ($\textsc{cd}_s$), and accumulated ($\textsc{ale}_s$). Specifically, we explain features that have more than 2 unique values to exclude one-hot-encoded features for which accumulated local effects are not intuitive to estimate. Finally, we use $\widehat{\textsc{dale}}_s$~\citep{gkolemis2023dale} to accurately estimate $\textsc{ale}_s$ for neural networks. 

\paragraph{Model perturbation.} We perform model parameter randomization tests~\citep{adebayo2018sanity} for global feature effect explanations. The idea is to sequentially perturb weights in consecutive layers of a neural network starting from the end. It was previously shown that gradient-based explanations from the class of local feature attributions are not significantly affected by such a perturbation, which might not be a desired property of an explanation method~\citep{adebayo2018sanity}. To implement model parameter randomization tests for the pretrained 3-layer neural networks, we add Gaussian noise $\mN(0, \sigma=0.5)$ to the weights. We repeat the test 20 times and visualize the average result with a standard error.

\paragraph{Results.} Figure~\ref{fig:exp2} shows results for all datasets. We observe, as expected, that feature effects are influenced by perturbing weights of a neural network. In many cases, (differential) accumulated local effects do not pass the model randomization test, i.e., are significantly less affected by drastic model perturbation than partial and conditional dependence. Our result is consistent with the work comparing removal-based to gradient-based local feature attributions~\citep{lin2023robustness}. In Appendix~\ref{app:experiments-model}, we provide more analogous results for different $\sigma$ values and report the drop in model predictive performance after perturbations.

\section{Conclusion}

We derived theoretical bounds for the robustness of feature effects to data and model perturbations. They give certain guarantees and intuition regarding how adversarial perturbations influence global explanations of ML models. Our theory can guide future work on improving these explanation methods to be more stable and faithful to the model and data distribution. We made several valuable connections to previous work, e.g., concerning the robustness of local feature attributions~\citep{lin2023robustness}, adversarial attacks on partial dependence~\citep{baniecki2022fooling}, and model parameter randomization tests for gradient-based explanations~\citep{adebayo2018sanity}.

Experimental results show that, on average, partial dependence is more robust to data perturbation than conditional dependence. Moreover, accumulated local effects do not pass the model randomization test, i.e., are significantly less affected by drastic model perturbation than partial and conditional dependence.

\paragraph{Limitations and future work.} Theorem~\ref{th:ale-data} assumes \emph{model $f$ is $L$-Lipschitz continuous} and future work can improve the bound to remove this assumption. It would say more about an explanation of a decision tree or, in general, a step-wise function that has infinite gradients not bounded by $L$. Theorems~\ref{th:ale-data}~\&~\ref{th:ale-model} are derived for the most popular case when $|s|=1$, but can be similarly derived for case when $|s|>1$. Our experiments are biased toward pretrained models from OpenXAI. Moreover, we acknowledge that the numerical approximation of total variation distance, as well as conditional distributions and effects, is prone to errors and might impact experimental results. Future theoretical and experimental work can analyze how feature dependence, e.g., correlation and interactions, impacts the robustness of global feature effects to model and data perturbation.

\begin{figure}[t!]
    \centering
    \includegraphics[width=0.49\textwidth]{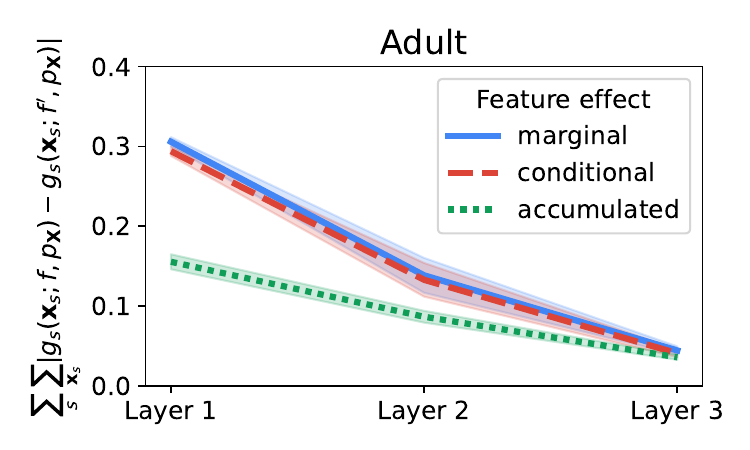}
    \includegraphics[width=0.49\textwidth]{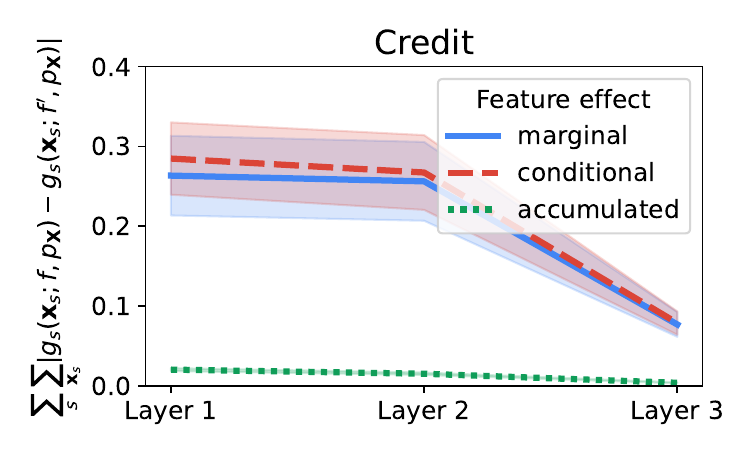}
    \includegraphics[width=0.49\textwidth]{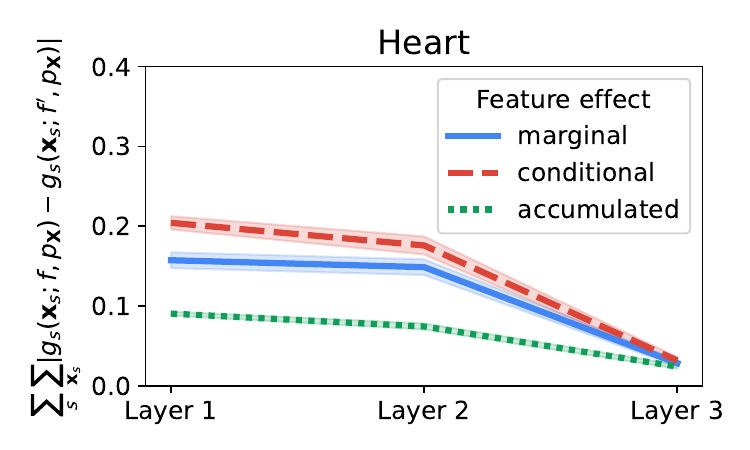}
    \includegraphics[width=0.49\textwidth]{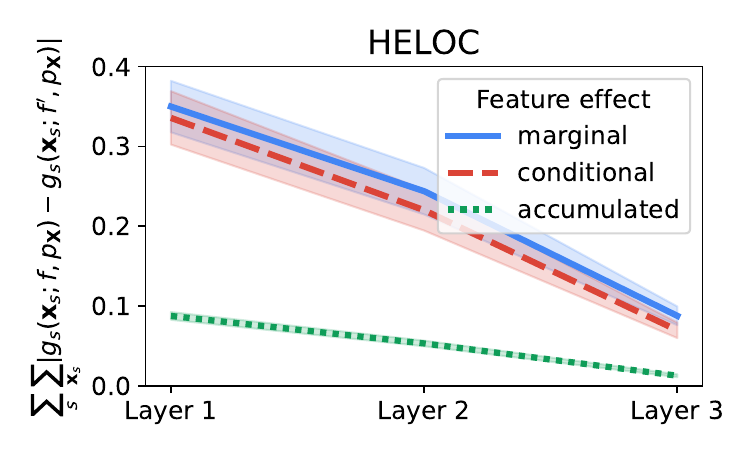}
    \includegraphics[width=0.49\textwidth]{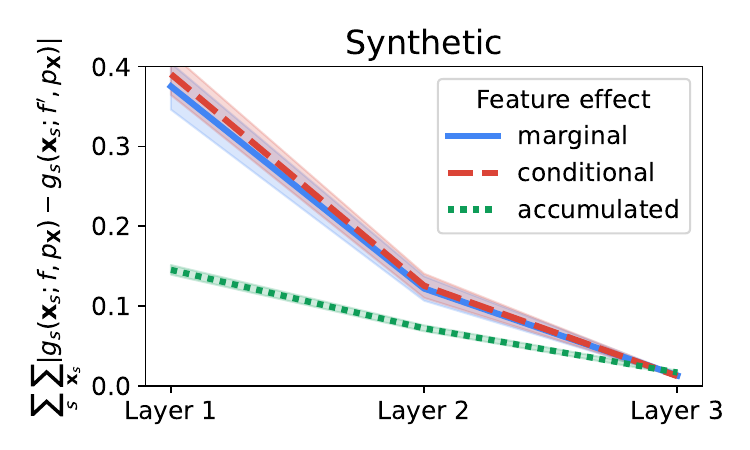}
    \includegraphics[width=0.49\textwidth]{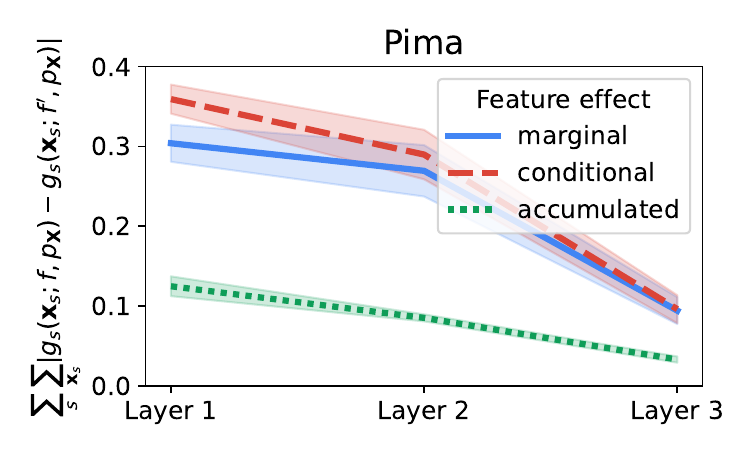}
    \caption{Robustness of feature effects to model perturbation for neural network models trained on six datasets. The X-axis denotes consecutive layers of a neural network being randomized sequentially. Values on the Y-axis are normalized per dataset and relate to the bounded distance between explanations in Lemma~\ref{lem:pd-model} and Theorem~\ref{th:ale-model}. Accumulated local effects do not pass the model randomization test, i.e., are significantly less affected by drastic model perturbation than partial (marginal) and conditional dependence.}
    \label{fig:exp2}
\end{figure}

\clearpage
\subsubsection*{Acknowledgements.} This work was financially supported by the Polish National Science Centre grant number 2021/43/O/ST6/00347.

%
\bibliographystyle{splncs04}
{\small\bibliography{references}}

\begin{thebibliography}{10}
\providecommand{\url}[1]{\texttt{#1}}
\providecommand{\urlprefix}{URL }
\providecommand{\doi}[1]{https://doi.org/#1}

\bibitem{adebayo2018sanity}
Adebayo, J., Gilmer, J., Muelly, M., Goodfellow, I., Hardt, M., Kim, B.: Sanity
  checks for saliency maps. In: NeurIPS (2018)

\bibitem{agarwal2022openxai}
Agarwal, C., Krishna, S., Saxena, E., Pawelczyk, M., Johnson, N., Puri, I.,
  Zitnik, M., Lakkaraju, H.: {OpenXAI: Towards a Transparent Evaluation of
  Model Explanations}. In: NeurIPS (2022)

\bibitem{apley2020visualizing}
Apley, D.W., Zhu, J.: {Visualizing the effects of predictor variables in black
  box supervised learning models}. Journal of the Royal Statistical Society:
  Series B (Statistical Methodology)  \textbf{82}(4),  1059--1086 (2020)

\bibitem{baniecki2024adversarial}
Baniecki, H., Biecek, P.: Adversarial attacks and defenses in explainable
  artificial intelligence: A survey. Information Fusion  \textbf{107},  102303
  (2024)

\bibitem{baniecki2022fooling}
Baniecki, H., Kretowicz, W., Biecek, P.: {Fooling Partial Dependence via Data
  Poisoning}. In: ECML PKDD (2022)

\bibitem{baniecki2023grammar}
Baniecki, H., Parzych, D., Biecek, P.: The grammar of interactive explanatory
  model analysis. Data Mining and Knowledge Discovery pp. 1--37 (2023)

\bibitem{bodria2023benchmarking}
Bodria, F., Giannotti, F., Guidotti, R., Naretto, F., Pedreschi, D.,
  Rinzivillo, S.: Benchmarking and survey of explanation methods for black box
  models. Data Mining and Knowledge Discovery pp. 1--60 (2023)

\bibitem{casalicchio2018visualizing}
Casalicchio, G., Molnar, C., Bischl, B.: {Visualizing the Feature Importance
  for Black Box Models}. In: ECML PKDD (2018)

\bibitem{freiesleben2023beyond}
Freiesleben, T., Grote, T.: Beyond generalization: a theory of robustness in
  machine learning. Synthese  \textbf{202}(4), ~109 (2023)

\bibitem{friedman2001greedy}
Friedman, J.H.: {Greedy Function Approximation: A Gradient Boosting Machine}.
  Annals of Statistics  \textbf{29}(5),  1189--1232 (2001)

\bibitem{gan2022is}
Gan, Y., Mao, Y., Zhang, X., Ji, S., Pu, Y., Han, M., Yin, J., Wang, T.: {``Is
  your explanation stable?'': A Robustness Evaluation Framework for Feature
  Attribution}. In: ACM CCS (2022)

\bibitem{ghorbani2019interpretation}
Ghorbani, A., Abid, A., Zou, J.: {Interpretation of Neural Networks is
  Fragile}. In: AAAI (2019)

\bibitem{gkolemis2023dale}
Gkolemis, V., Dalamagas, T., Diou, C.: {DALE: Differential Accumulated Local
  Effects for efficient and accurate global explanations}. In: ACML (2023)

\bibitem{gkolemis2023rhale}
Gkolemis, V., Dalamagas, T., Ntoutsi, E., Diou, C.: {RHALE: Robust and
  Heterogeneity-aware Accumulated Local Effects}. In: ECAI (2023)

\bibitem{br2018simple}
Greenwell, B.M., Boehmke, B.C., McCarthy, A.J.: A simple and effective
  model-based variable importance measure. arXiv preprint arXiv:1805.04755
  (2018)

\bibitem{guidotti2022stable}
Guidotti, R., Monreale, A., Ruggieri, S., Naretto, F., Turini, F., Pedreschi,
  D., Giannotti, F.: Stable and actionable explanations of black-box models
  through factual and counterfactual rules. Data Mining and Knowledge Discovery
  pp. 1--38 (2022)

\bibitem{guyomard2023generating}
Guyomard, V., Fessant, F., Guyet, T., Bouadi, T., Termier, A.: {Generating
  Robust Counterfactual Explanations}. In: ECML PKDD (2023)

\bibitem{hedstrom2023quantus}
Hedstrom, A., Weber, L., Krakowczyk, D., Bareeva, D., Motzkus, F., Samek, W.,
  Lapuschkin, S., Hohne, M.M.C.: {Quantus: An Explainable AI Toolkit for
  Responsible Evaluation of Neural Network Explanations and Beyond}. Journal of
  Machine Learning Research  \textbf{24}(34),  1--11 (2023)

\bibitem{herbinger2022repid}
Herbinger, J., Bischl, B., Casalicchio, G.: {REPID: Regional Effect Plots with
  implicit Interaction Detection}. In: AISTATS (2022)

\bibitem{hooker2007generalized}
Hooker, G.: Generalized functional anova diagnostics for high-dimensional
  functions of dependent variables. Journal of Computational and Graphical
  Statistics  \textbf{16}(3),  709--732 (2007)

\bibitem{huang2023safari}
Huang, W., Zhao, X., Jin, G., Huang, X.: {SAFARI: Versatile and Efficient
  Evaluations for Robustness of Interpretability}. In: ICCV (2023)

\bibitem{jia2021studying}
Jia, Y., Frank, E., Pfahringer, B., Bifet, A., Lim, N.: {Studying and
  Exploiting the Relationship Between Model Accuracy and Explanation Quality}.
  In: ECML PKDD (2021)

\bibitem{kemter2023controls}
Kemter, M., Marwan, N., Villarini, G., Merz, B.: {Controls on Flood Trends
  Across the United States}. Water Resources Research  \textbf{59}(2),
  e2021WR031673 (2023)

\bibitem{laberge2023fooling}
Laberge, G., A{\"\i}vodji, U., Hara, S., Marchand, M., Khomh, F.: {Fooling SHAP
  with Stealthily Biased Sampling}. In: ICLR (2023)

\bibitem{lakkaraju2020robust}
Lakkaraju, H., Arsov, N., Bastani, O.: {Robust and Stable Black Box
  Explanations}. In: ICML (2020)

\bibitem{lin2023robustness}
Lin, C., Covert, I., Lee, S.I.: {On the Robustness of Removal-Based Feature
  Attributions}. In: NeurIPS (2023)

\bibitem{lundberg2017unified}
Lundberg, S.M., Lee, S.I.: {A Unified Approach to Interpreting Model
  Predictions}. In: NeurIPS (2017)

\bibitem{mangalathu2022machine}
Mangalathu, S., Karthikeyan, K., Feng, D.C., Jeon, J.S.: Machine-learning
  interpretability techniques for seismic performance assessment of
  infrastructure systems. Engineering Structures  \textbf{250},  112883 (2022)

\bibitem{mayer2023minimizing}
Meyer, A.P., Ley, D., Srinivas, S., Lakkaraju, H.: {On Minimizing the Impact of
  Dataset Shifts on Actionable Explanations}. In: UAI (2023)

\bibitem{molnar2023relating}
Molnar, C., et~al.: Relating the partial dependence plot and permutation
  feature importance to the data generating process. In: XAI (2023)

\bibitem{muschalik2023ipdp}
Muschalik, M., Fumagalli, F., Jagtani, R., Hammer, B., Hüllermeier, E.: {iPDP:
  On Partial Dependence Plots in Dynamic Modeling Scenarios}. In: XAI (2023)

\bibitem{noppel2024explainable}
Noppel, M., Wressnegger, C.: {SoK: Explainable Machine Learning in Adversarial
  Environments}. In: IEEE S\&P (2024)

\bibitem{petch2022opening}
Petch, J., Di, S., Nelson, W.: {Opening the Black Box: The Promise and
  Limitations of Explainable Machine Learning in Cardiology}. Canadian Journal
  of Cardiology  \textbf{38}(2),  204--213 (2022)

\bibitem{ribeiro2016why}
Ribeiro, M.T., Singh, S., Guestrin, C.: {``Why Should I Trust You?'':
  Explaining the Predictions of Any Classifier}. In: KDD (2016)

\bibitem{robertson2023gut}
Robertson, R.C., et~al.: The gut microbiome and early-life growth in a
  population with high prevalence of stunting. Nature Communications
  \textbf{14}(1), ~654 (2023)

\bibitem{schwalbe2023comprehensive}
Schwalbe, G., Finzel, B.: A comprehensive taxonomy for explainable artificial
  intelligence: a systematic survey of surveys on methods and concepts. Data
  Mining and Knowledge Discovery pp. 1--59 (2023)

\bibitem{slack2020fooling}
Slack, D., Hilgard, S., Jia, E., Singh, S., Lakkaraju, H.: {Fooling LIME and
  SHAP: Adversarial Attacks on Post hoc Explanation Methods}. In: AIES (2020)

\bibitem{virmaux2018lipschitz}
Virmaux, A., Scaman, K.: Lipschitz regularity of deep neural networks: analysis
  and efficient estimation. In: NeurIPS (2018)

\bibitem{wicker2023robust}
Wicker, M.R., Heo, J., Costabello, L., Weller, A.: {Robust Explanation
  Constraints for Neural Networks}. In: ICLR (2023)

\bibitem{zien2009feature}
Zien, A., Kr{\"a}mer, N., Sonnenburg, S., R{\"a}tsch, G.: The feature
  importance ranking measure. In: ECML PKDD (2009)

\end{thebibliography}

\clearpage
\appendix

Here, we provide an additional description of methods and results regarding experiments done in Section~\ref{sec:experiments}.

\section{Experiments: robustness to data perturbation}\label{app:experiments-data}

\paragraph{Data perturbation.} We constraint to perturbing only the top 2 most important features (in the case when $s$ is the most important, we perturb the 2nd and 3rd important) as measured by the variance of partial dependence. 

In random perturbation, we add (once) univariate Gaussian noise $\mN(0, \sigma)$ with $\sigma = \{ 0.01, 0.05, 0.10,\allowbreak0.12, 0.25\}$ to each of the perturbed features. 

In adversarial perturbation, a genetic algorithm performs mutation, crossover, evaluation, and selection between a population of 100 individuals (dataset instances) for 200 iterations. 
In each iteration, mutation adds univariate Gaussian noise $\mN(0, \sigma)$ with $\sigma = \{ 0.01, 0.05, 0.10, 0.25, 0.33\}$ to each of the perturbed features. 
It always checks if any new value is out of distribution (edges of the domain of a particular feature) and if so, samples a new value from the original distribution. This is to constrain the perturbation to the data manifold. 

A crossover operator exchanges values in corresponding row/column indices ($i,j$) of the dataset between the two parent individuals to generate new child individuals. 
Evaluation of individuals considers calculating a fitness function, which here is a distance between the original explanation value (e.g., $0.53$) and a target (in this case $0$ or $1$). 
Finally, the algorithm uses a standard roulette wheel selection to choose individuals for the next iteration. For further details on the method, refer to the original article~\citep{baniecki2022fooling}. We set the remaining hyperparameters to default values.

We repeat random and adversarial perturbations 5 times and visualize all the obtained results.

\paragraph{Additional results.} Figure~\ref{fig:exp1_heloc} shows results of the first experiment for the HELOC dataset. On average in our setup, partial dependence~(marginal) is more robust to data perturbation than conditional dependence

\begin{figure}[ht]
    \centering
    \includegraphics[width=0.99\textwidth]{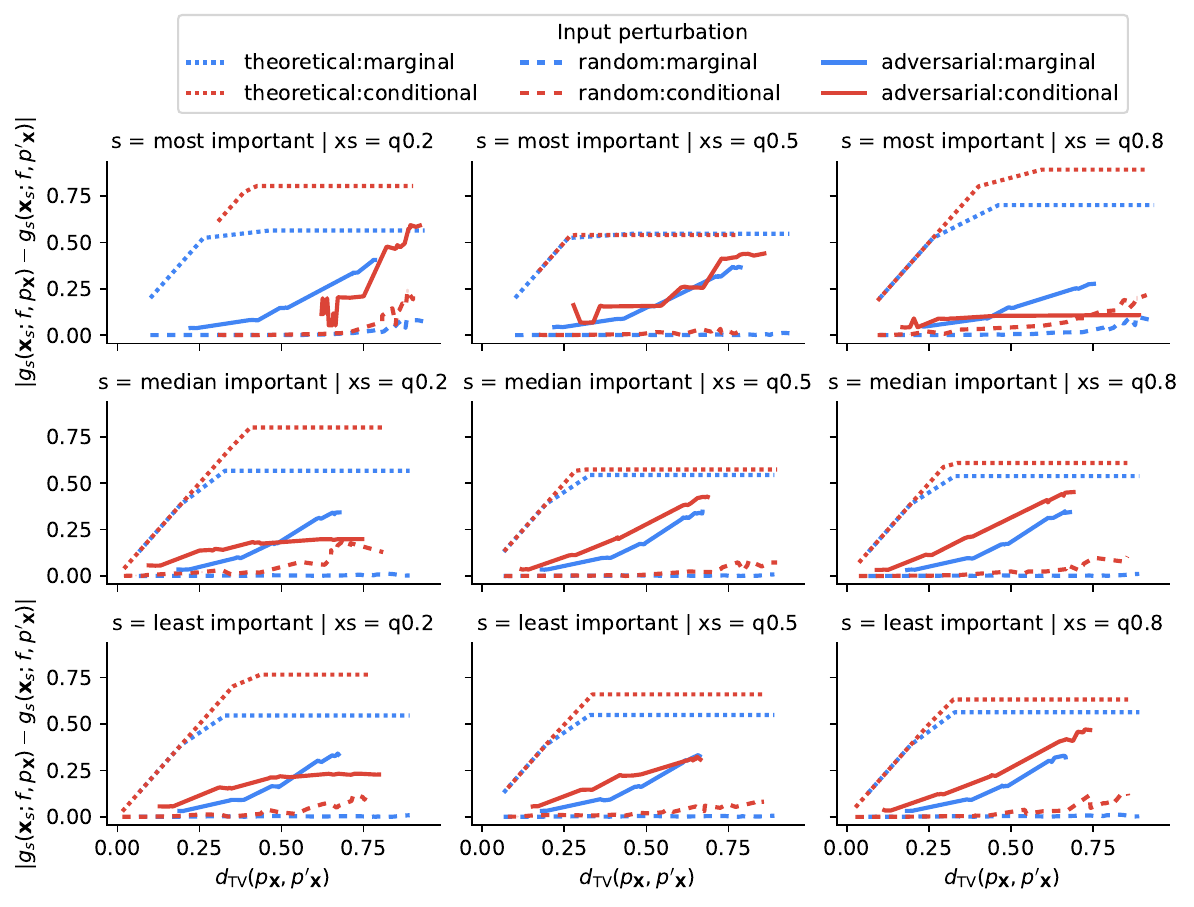}
    \caption{Robustness of feature effects to data perturbation for a neural network trained on the \textbf{HELOC} dataset. Feature effect is evaluated for a feature $s$ (rows) and quantile value $\bxs$ (columns). The dotted lines denote the approximated theoretical bound derived in Theorem~\ref{th:pd-data}, which ends with a case-specific constant value as described in Remark~\ref{remark:constant}. Results are with respect to estimating feature effects on marginal $\pX$ (partial dependence) or conditional distribution $\pXsbXs$ (conditional dependence).}
    \label{fig:exp1_heloc}
\end{figure}

\section{Experiments: robustness to model perturbation}\label{app:experiments-model}

\paragraph{Additional results.} Figures~\ref{fig:exp2_adult}, \ref{fig:exp2_credit}, \ref{fig:exp2_heart}, \ref{fig:exp2_heloc}, \ref{fig:exp2_synthetic} \& \ref{fig:exp2_pima} show additional results of the second experiment for all the datasets. We can observe how different $\sigma$ values impact the model parameter randomization test. For a broader context, we report the drop in model performance (accuracy) after each layer is sequentially perturbed. Clearly in cases where parameter perturbations are significant enough to impact model performance, (differential) accumulated local effects remain more robust (here, in a bad way) than partial and conditional dependence. Our result is consistent with the original work introducing model randomization tests for saliency maps~\citep{adebayo2018sanity}, as well as the work comparing removal-based to gradient-based local feature attributions~\citep{lin2023robustness}.

\begin{figure}
    \centering
    \includegraphics[width=0.49\textwidth]{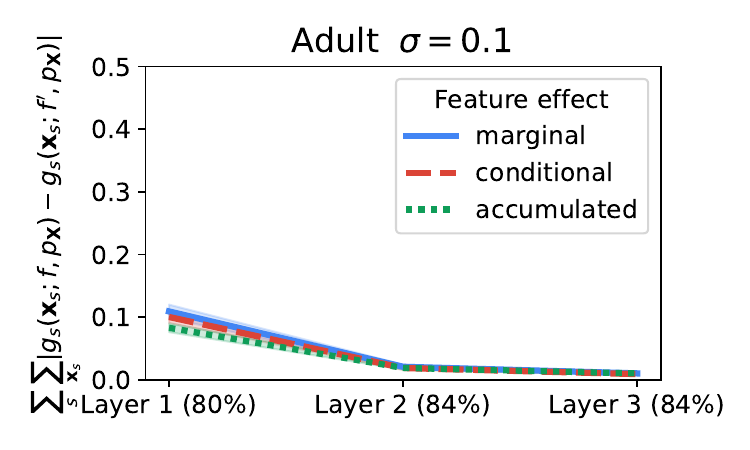}
    \includegraphics[width=0.49\textwidth]{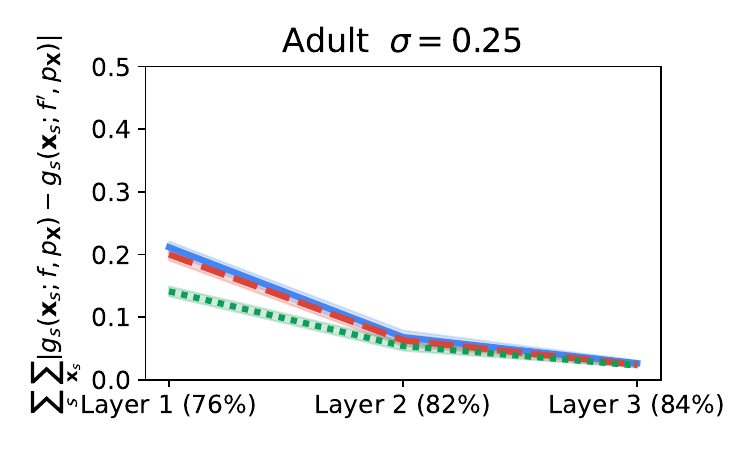}
    \includegraphics[width=0.49\textwidth]{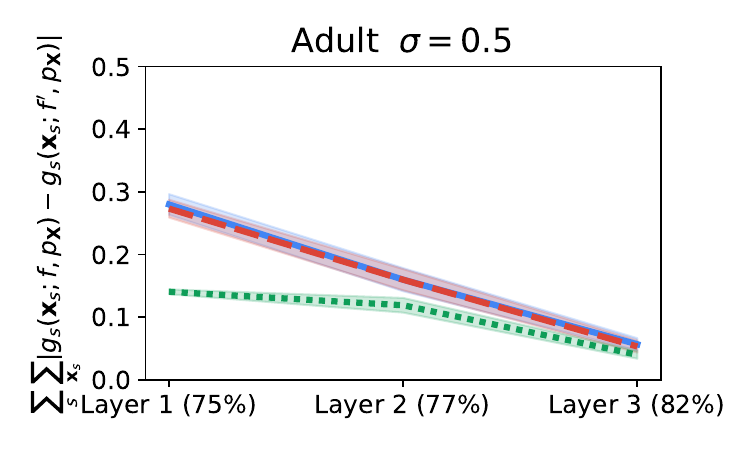}
    \includegraphics[width=0.49\textwidth]{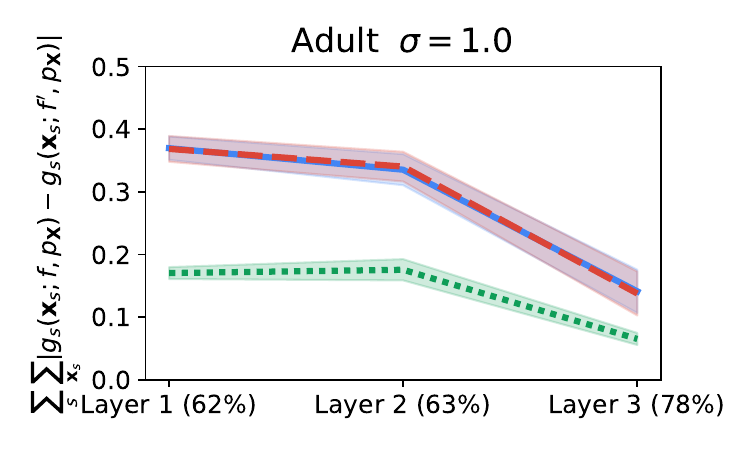}
    \caption{Robustness of feature effects to model perturbation for a neural network trained on the \textbf{Adult} dataset (85\% accuracy). The X-axis denotes consecutive layers of a neural network being randomized sequentially. Values in parentheses report the accuracy of the model after each layer is perturbed. Values on the Y-axis are normalized per dataset and relate to the bounded distance between explanations in Lemma~\ref{lem:pd-model} and Theorem~\ref{th:ale-model}.}
    \label{fig:exp2_adult}
\end{figure}

\begin{figure}
    \centering
    \includegraphics[width=0.49\textwidth]{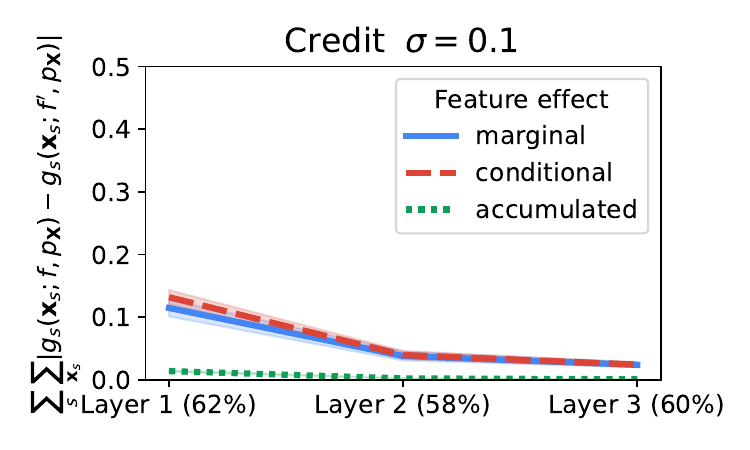}
    \includegraphics[width=0.49\textwidth]{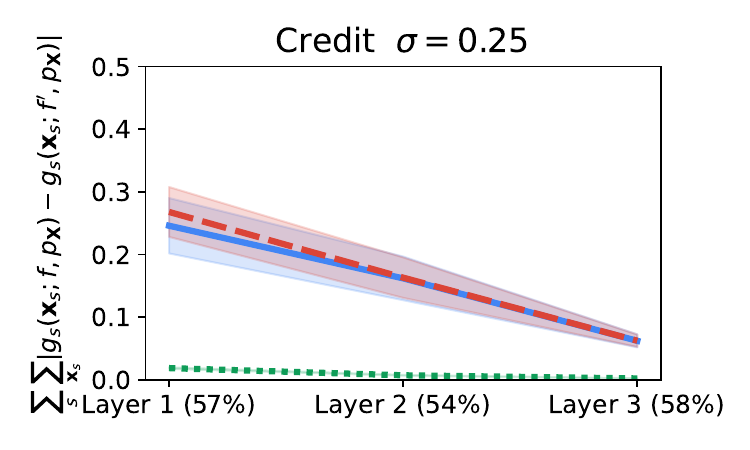}
    \includegraphics[width=0.49\textwidth]{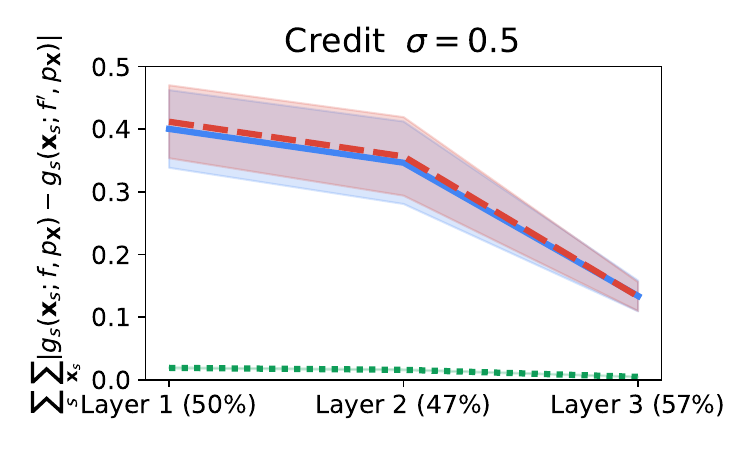}
    \includegraphics[width=0.49\textwidth]{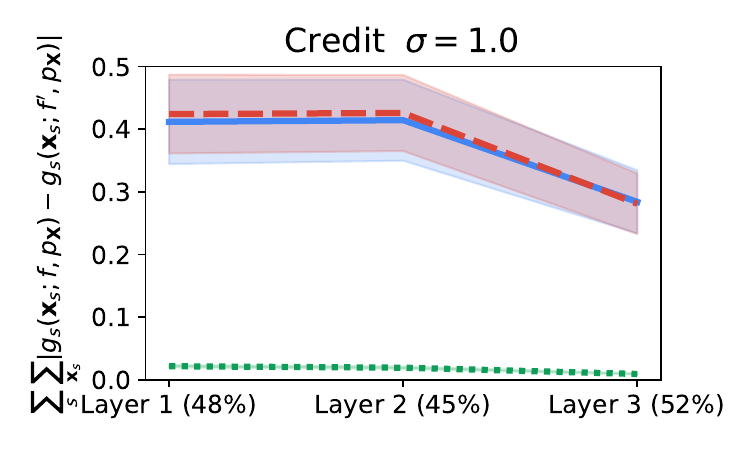}
    \caption{Robustness of feature effects to model perturbation for a neural network trained on the \textbf{Credit} dataset (75\% accuracy). Values in parentheses report the accuracy of the model after each layer is perturbed.}
    \label{fig:exp2_credit}
\end{figure}

\begin{figure}
    \centering
    \includegraphics[width=0.49\textwidth]{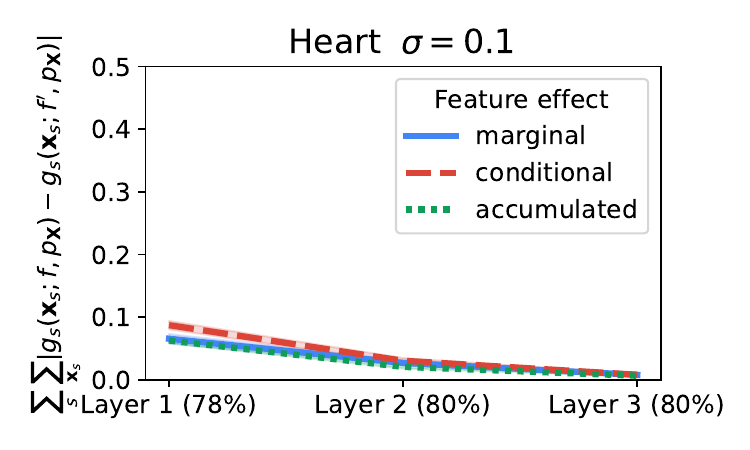}
    \includegraphics[width=0.49\textwidth]{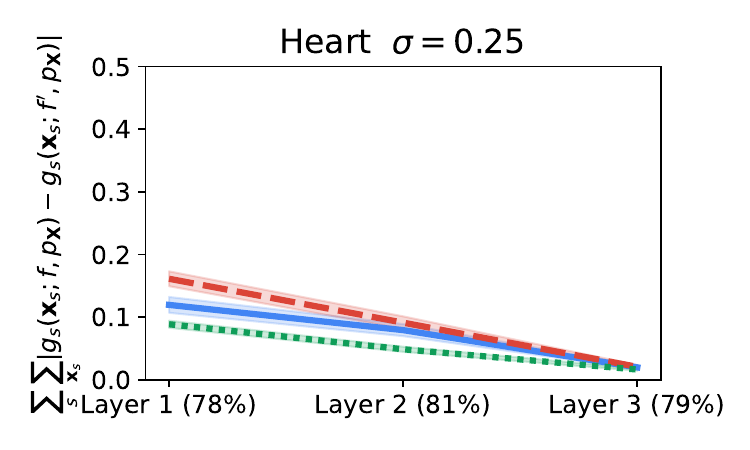}
    \includegraphics[width=0.49\textwidth]{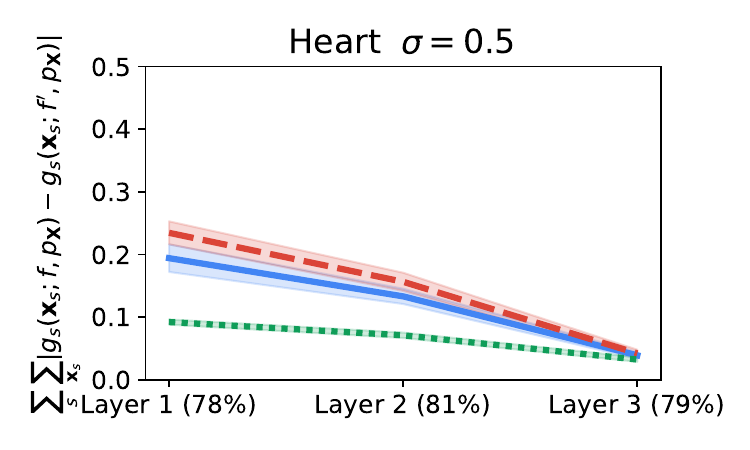}
    \includegraphics[width=0.49\textwidth]{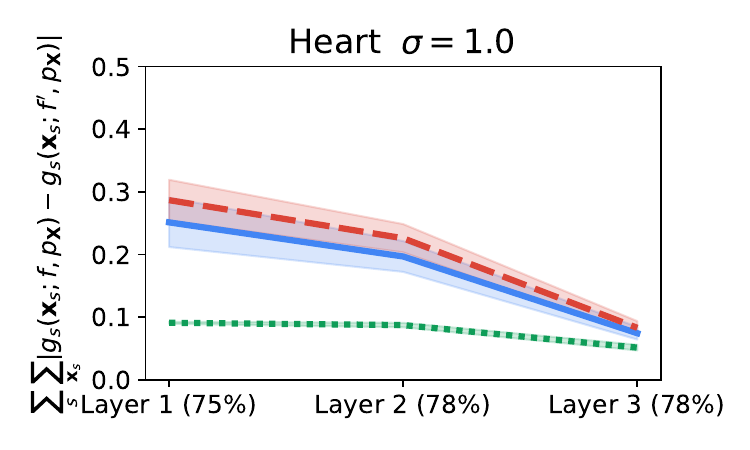}
    \caption{Robustness of feature effects to model perturbation for a neural network trained on the \textbf{Heart} dataset (85\% accuracy). Values in parentheses report the accuracy of the model after each layer is perturbed. The X-axis denotes consecutive layers of a neural network being randomized sequentially. Values in parentheses report the accuracy of the model after a layer is perturbed. Values on the Y-axis are normalized per dataset and relate to the bounded distance between explanations in Lemma~\ref{lem:pd-model} and Theorem~\ref{th:ale-model}.}
    \label{fig:exp2_heart}
\end{figure}

\begin{figure}
    \centering
    \includegraphics[width=0.49\textwidth]{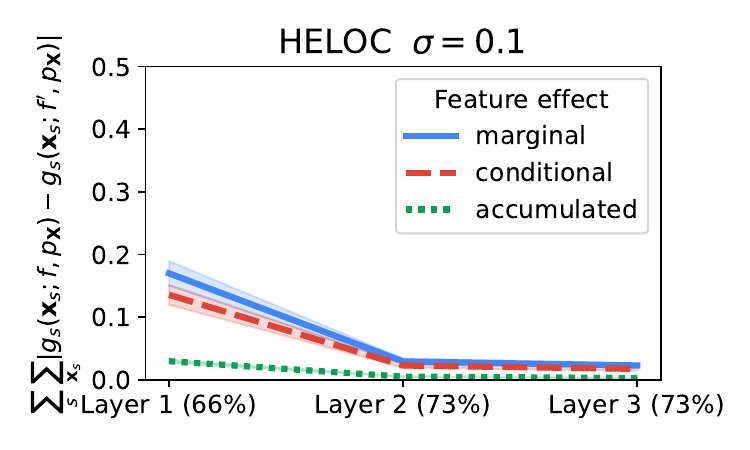}
    \includegraphics[width=0.49\textwidth]{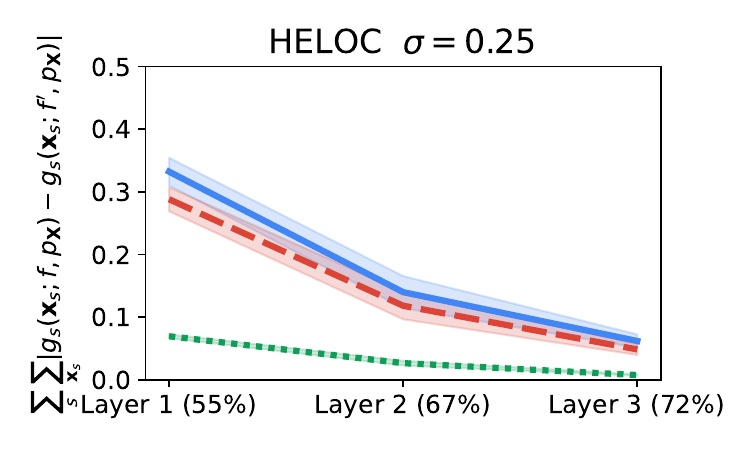}
    \includegraphics[width=0.49\textwidth]{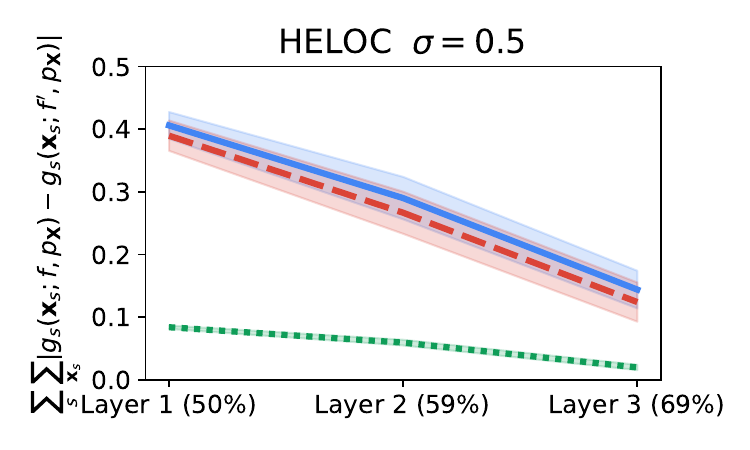}
    \includegraphics[width=0.49\textwidth]{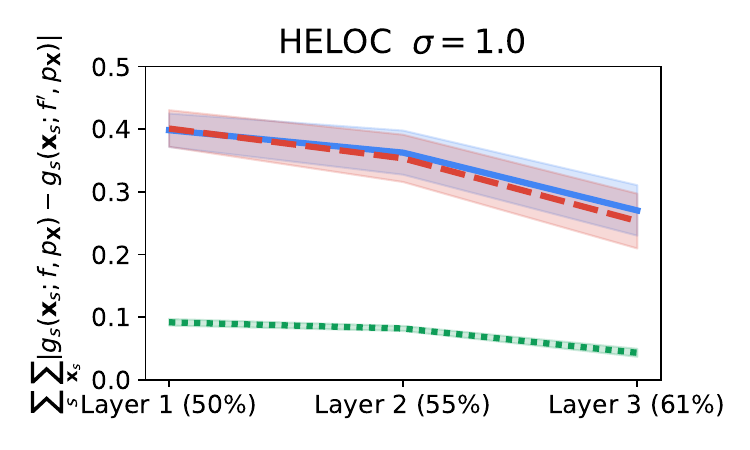}
    \caption{Robustness of feature effects to model perturbation for a neural network trained on the \textbf{HELOC} dataset (74\% accuracy). Values in parentheses report the accuracy of the model after each layer is perturbed.}
    \label{fig:exp2_heloc}
\end{figure}

\begin{figure}
    \centering
    \includegraphics[width=0.49\textwidth]{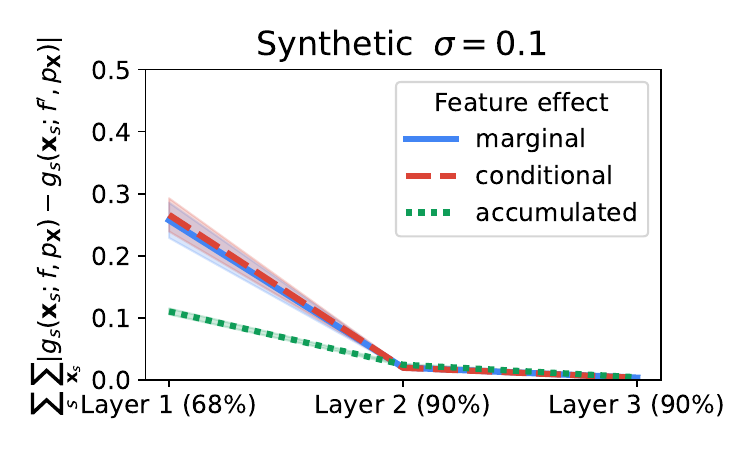}
    \includegraphics[width=0.49\textwidth]{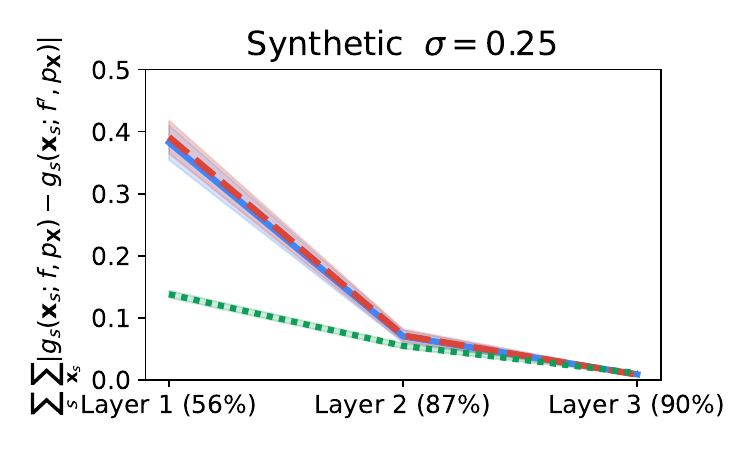}
    \includegraphics[width=0.49\textwidth]{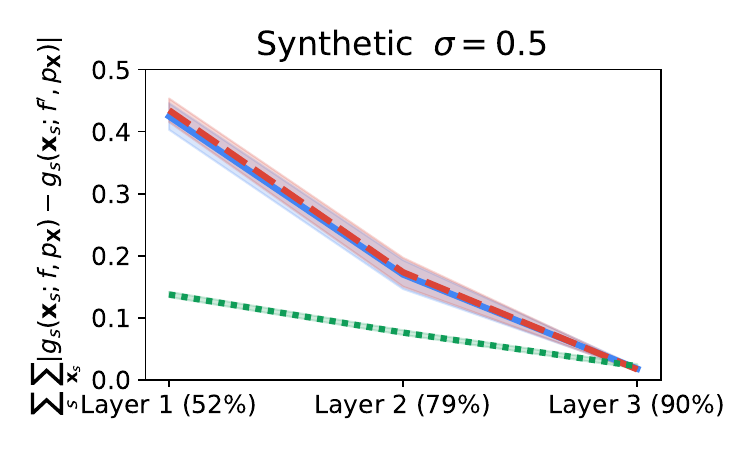}
    \includegraphics[width=0.49\textwidth]{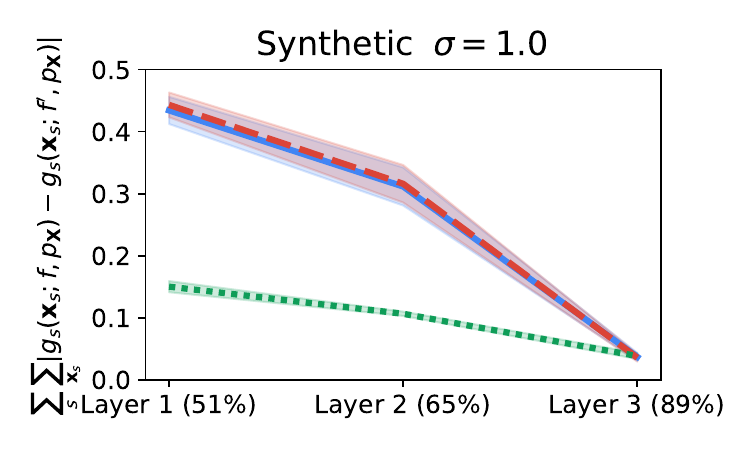}
    \caption{Robustness of feature effects to model perturbation for a neural network trained on the \textbf{Synthetic} dataset (92\% accuracy). The X-axis denotes consecutive layers of a neural network being randomized sequentially. Values in parentheses report the accuracy of the model after each layer is perturbed. Values on the Y-axis are normalized per dataset and relate to the bounded distance between explanations in Lemma~\ref{lem:pd-model} and Theorem~\ref{th:ale-model}.}
    \label{fig:exp2_synthetic}
\end{figure}

\begin{figure}
    \centering
    \includegraphics[width=0.49\textwidth]{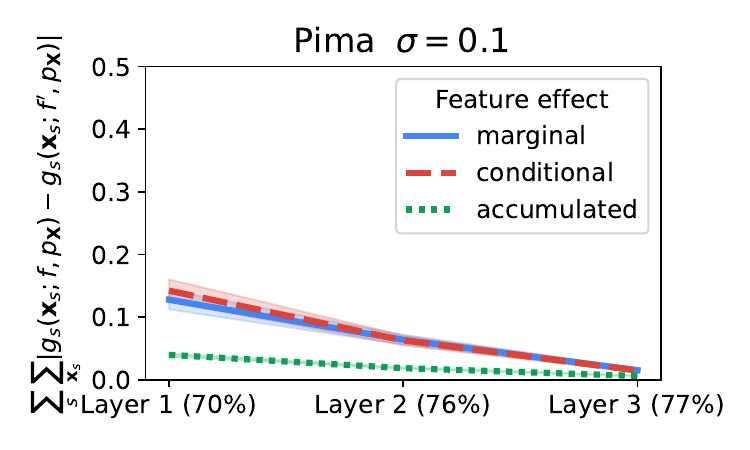}
    \includegraphics[width=0.49\textwidth]{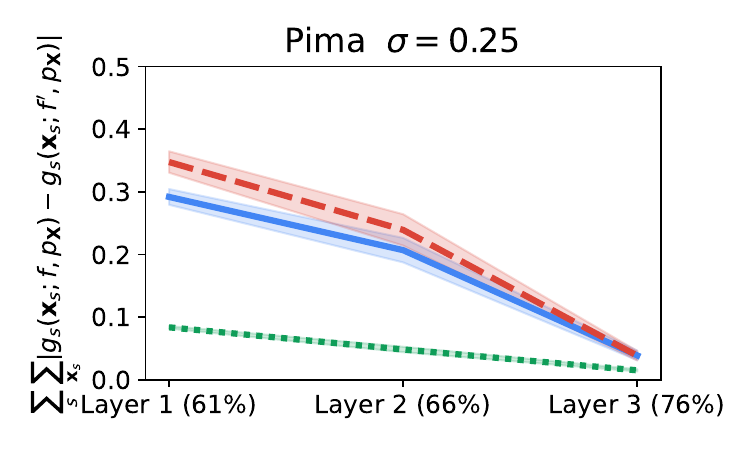}
    \includegraphics[width=0.49\textwidth]{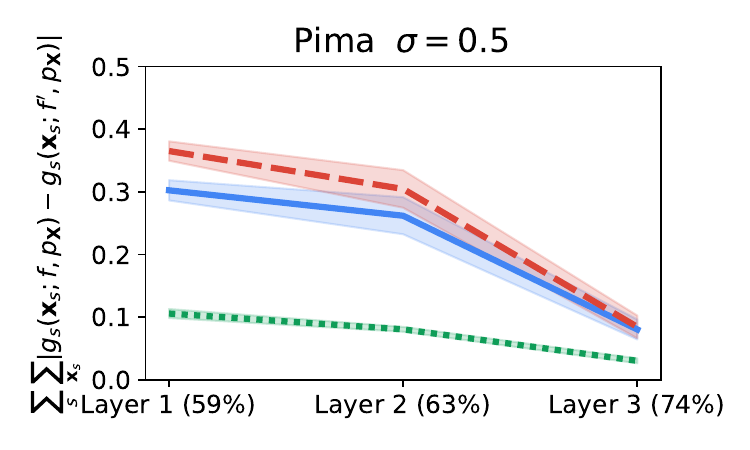}
    \includegraphics[width=0.49\textwidth]{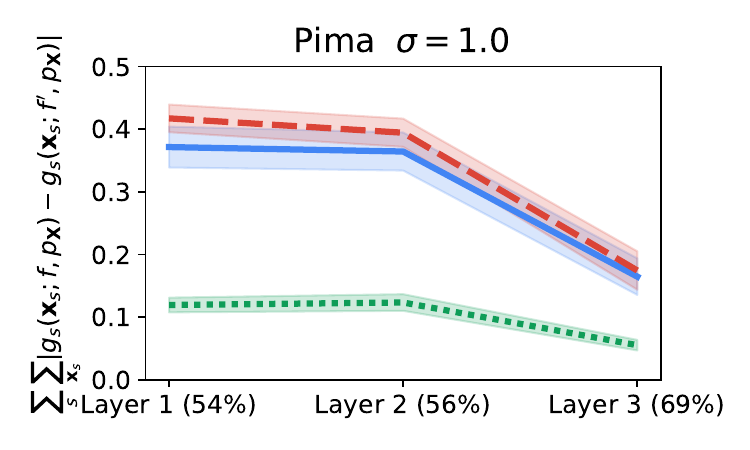}
    \caption{Robustness of feature effects to model perturbation for a neural network trained on the \textbf{Pima} dataset (77\% accuracy). Values in parentheses report the accuracy of the model after each layer is perturbed.}
    \label{fig:exp2_pima}
\end{figure}

\end{document}